\documentclass[twoside,11pt]{article}

\usepackage[T1]{fontenc}
\usepackage{amssymb}
\usepackage{amsmath}
\usepackage{amsfonts}
\usepackage{amsthm}

\usepackage{a4wide}

\usepackage{natbib}
\usepackage{graphicx}

\usepackage{bm} 

\usepackage{algorithmic} 
\usepackage{booktabs} 

\newcommand{\trace}{{\rm Tr}}

\newcommand{\Grassmann}{{\mathrm{Gr}(r,d)}}
\newcommand{\Cone}[1]{S_+({#1})}
\newcommand{\FixedRank}{S_+(r,d)}
\newcommand{\Stiefel}{{\mathrm{St}(r,d)}}

\newcommand{\OG}[1]{{\mathcal{O}({#1})}}
\newcommand{\GL}[1]{{\mathrm{GL}({#1})}}

\newcommand{\R}{R}
\newcommand{\Exp}{{\mathrm{Exp}}}
\newcommand{\Sym}{{\mathrm{Sym}}}

\newcommand{\qf}{\mathrm{qf}}

\newcommand{\Proj}{\Pi}

\newcommand{\set}[2]{{\{{#1}\ \mathrm{s.t.}\ {#2}\}}}

\newcommand{\argmin}{\operatornamewithlimits{arg\,min}}

\newcommand{\grad}[2]{\mathrm{grad}_{{#1}}{{#2}}}

\newcommand{\mat}[1]{{\bf #1}}
\renewcommand{\vec}[1]{{\bf #1}}

\newtheorem{proposition}{Proposition} [section]

\newcommand{\footnoteremember}[2]{\footnote{#2}\newcounter{#1}\setcounter{#1}{\value{footnote}}}
\newcommand{\footnoterecall}[1]{\footnotemark[\value{#1}]}

\title{Regression on Fixed-Rank Positive Semidefinite Matrices:\\
     a Riemannian Approach}

\author{Gilles Meyer\footnoteremember{ulg}{Department of Electrical Engineering and Computer Science, University of Li\`{e}ge, B-4000 Li\`{e}ge, Belgium,\hfill\texttt{\{g.meyer,r.sepulchre\}@ulg.ac.be}.}\and Silv\`{e}re Bonnabel\footnote{Robotics center, Mines ParisTech, Boulevard Saint-Michel, 60, 75272 Paris, France,\hfill\texttt{silvere.bonnabel@mines-paristech.fr}} \and Rodolphe Sepulchre\footnoterecall{ulg}}

\begin{document}
\maketitle
\section*{Abstract}
The paper addresses the problem of learning a regression model parameterized by a fixed-rank positive semidefinite matrix. The focus is on the nonlinear nature of the search space and on scalability to high-dimensional problems. The mathematical developments rely on the theory of gradient descent algorithms adapted to the Riemannian geometry that underlies the set of fixed-rank positive semidefinite matrices. In contrast with previous contributions in the literature, no restrictions are imposed on the range space of the learned matrix. The resulting algorithms maintain a linear complexity in the problem size and enjoy important invariance properties. We apply the proposed algorithms to the problem of learning a distance function parameterized by a positive semidefinite matrix. Good performance is observed on classical benchmarks.

\section{Introduction}
A fundamental problem of machine learning is the learning of a distance between data samples. When the distance can be written as a quadratic form (either in the data space (Mahalanobis distance) or in a kernel feature space (kernel distance)), the learning problem is a regression problem on the set of positive definite matrices. The regression problem is turned into the minimization of the prediction error, leading to an optimization framework and gradient-based algorithms.

The present paper focuses on the nonlinear nature of the search space. The classical framework of gradient-based learning can be generalized provided that the nonlinear search space is equipped with a proper Riemannian geometry. Adopting this general framework, we design novel learning algorithms on the space of fixed-rank positive semidefinite matrices, denoted by $\FixedRank$, where $d$ is the dimension of the matrix, and $r$ is its rank. Learning a parametric model in $\FixedRank$ amounts to jointly learn a $r$-dimensional subspace and a quadratic distance in this subspace.

The framework is motivated by {\em low-rank learning} in large-scale applications. If the data space is of dimension $d$, the goal is to maintain a linear computational complexity $O(d)$. In contrast to the classical approach of first reducing the dimension of the data and then learning a distance in the reduced space, there is an obvious conceptual advantage to perform the two tasks simultaneously. If this objective can be achieved without increasing the numerical cost of the algorithm, the advantage becomes also practical.

Our approach makes use of two quotient geometries of the set $\FixedRank$ that have been recently studied by \cite{journee09a} and \cite{bonnabel09a}. Making use of a general theory of line-search algorithms in quotient matrix spaces \citep{absil08a}, we obtain concrete gradient updates that maintain the rank and the positivity of the learned model at each iteration. This is because the update is intrinsically constrained to belong to the nonlinear search space, in contrast to early learning algorithms that neglect the non linear nature of the search space in the update and impose the constraints a posteriori \citep{xing02a,globerson05a}.

Not surprisingly, our approach has close connections with a number of recent contributions on learning algorithms. Learning problems over nonlinear matrix spaces include the learning of subspaces \citep{crammer06a,warmuth07a}, rotation matrices \citep{arora09a}, and positive definite matrices \citep{tsuda05a}. The space of (full-rank) positive definite matrices $\Cone{d}$ is of particular interest since it coincides with our set of interest in the particular case $r=d$.

The use of Bregman divergences and alternating projection has been recently investigated for learning in $\Cone{d}$. \cite{tsuda05a} propose to use the {\em von Neumann} divergence, resulting in a generalization of the well-known AdaBoost algorithm~\citep{schapire99a} to positive definite matrices. The use of the so-called {\em LogDet} divergence has also been investigated by \cite{davis07a} in the context of Mahalanobis distance learning. 

More recently, algorithmic work has focused on scalability in terms of dimensionality and data set size. A natural extension of the previous work on positive definite matrices is thus to consider low-rank positive semidefinite matrices. Indeed, whereas algorithms based on full-rank matrices scale as $O(d^3)$ and require $O(d^2)$ storage units, algorithms based on low-rank matrices scale as $O(dr^2)$ and require $O(dr)$ storage units \citep{fine01a,bach05a}. This is a significant complexity reduction as the approximation rank $r$ is typically very small compared to the dimension of the problem $d$. 

Extending the work of \cite{tsuda05a}, \cite{kulis09a} recently considered the learning of positive semidefinite matrices. The authors consider Bregman divergence measures that enjoy convexity properties and lead to updates that preserve the rank as well as the positive semidefinite property. However, these divergence-based algorithms intrinsically constrain the learning algorithm to a fixed range space. A practical limitation of this approach is that the subspace of the learned matrix is fixed beforehand by the initial condition of the algorithm.

The approach proposed in the present paper is in a sense more classical (we just perform a line-search in a Riemannian manifold) but we show how to interpret Bregman divergence based algorithms in our framework. This is potentially a contribution of independent interest since a general convergence theory exists for line-search algorithms on Riemannian manifolds. The generality of the proposed framework is of course motivated by the non-convex nature of the rank constraint.

The paper is organized as follows. Section \ref{sec:regression-riemannian-space} presents the general optimization framework of Riemannian learning. This framework is then applied to the learning of subspaces (Section \ref{sec:grassmann}), positive definite matrices (Section \ref{sec:cone}) and fixed-rank positive semidefinite matrices (Section \ref{sec:fixed-rank}). The novel proposed algorithms are presented in Section \ref{sec:algorithms}. Section \ref{sec:discussion} discusses the relationship to existing work as well as extensions of the proposed approach. Applications are presented in Section \ref{sec:applications} and experimental results are presented in Section~\ref{sec:experiments}.

\section{Linear Regression on Riemannian Spaces}\label{sec:regression-riemannian-space}
We consider the following standard regression problem. Given
\begin{enumerate}
	\renewcommand{\theenumi}{\roman{enumi}}
	\renewcommand{\labelenumi}{(\theenumi)}
	\item data points $\mat{X}$, in a linear data space $\mathcal{X}=\mathbb{R}^{d\times d}$,
	\item observations $y$, in a linear output space $\mathcal{Y}=\mathbb{R}$, (or $\mathbb{R}^d$),
	\item a regression model $\hat{y}=\hat{y}_\mat{W}(\mat{X})$ parameterized by a matrix $\mat{W}$ in a search space $\mathcal{W}$,
	\item a quadratic loss function $\ell(\hat{y},y) = \frac{1}{2}(\hat{y} - y)^2$,
\end{enumerate}
find the optimal fit $\mat{W}^*$ that minimizes the {\em expected cost}
\begin{equation*}\label{eq:expected-loss}
	F(\mat{W}) = \mathbb{E}_{\mat{X},y}\{\ell(\hat{y},y)\} 
						 = \int \ell(\hat{y},y)\ dP(\mat{X},y),
\end{equation*}
where $\ell(\hat{y},y)$ penalizes the discrepancy between observations and predictions, and $P(\mat{X},y)$ is the (unknown) joint probability distribution over data and observation pairs. Although our main interest will be in the scalar model
\begin{equation*}
	\hat{y} = \trace(\mat{W}\mat{X}),
\end{equation*}
the theory applies equally to vector data points $\vec{x}\in\mathbb{R}^{d}$, $\hat{y}=\trace(\mat{W}\vec{x}\vec{x}^T)=\vec{x}^T\mat{W}\vec{x}$, to a regression model parameterized by a vector $\vec{w}\in\mathbb{R}^d$, $\hat{y}=\vec{w}^T\vec{x}$, or to a vector output space $\hat{y}=\mat{W}\vec{x}$.

As it is generally not possible to compute $F(\mat{W})$ explicitly, batch learning algorithms minimize instead the {\em empirical cost}
\begin{equation}\label{eq:empirical-cost}
	f_{n}(\mat{W}) = \frac{1}{2n} \sum_{i = 1}^{n} (\hat{y}_i - y_i)^2,
\end{equation}
which is the average loss computed over a finite number of samples $\{(\mat{X}_i,y_i)\}_{i=1}^n$.

Online learning algorithms \citep{bottou04a} consider possibly infinite sets of samples $\{(\mat{X}_t,y_t)\}_{t\geq 1}$, received one at a time. At time $t$, the online learning algorithm minimizes the instantaneous cost
\begin{equation*}
	f_{t}(\mat{W}) = \frac{1}{2}(\hat{y}_t - y_t)^2.
\end{equation*}
In the sequel, we only present online versions of algorithms to shorten the exposition. The single necessary change to convert an online algorithm into its batch counterpart is to perform, at each iteration, the minimization of the empirical cost $f_{n}$ instead of the minimization of the instantaneous cost $f_{t}$. In the sequel, we denote by $f$ the cost function that is minimized at each iteration.

Our focus will be on {\em nonlinear} search spaces $\mathcal{W}$. We only require $\mathcal{W}$ to have the structure of a Riemannian matrix manifold. Following \cite{absil08a}, an abstract gradient descent algorithm can then be derived based on the update formula
\begin{equation}\label{eq:linesearch-manifold}
	\mat{W}_{t+1} = \R_{\mat{W}_t}(-s_t\ \grad{}{f(\mat{W}_t)}).
\end{equation}
The gradient $\grad{}{f(\mat{W}_t)}$ is an element of the tangent space $T_{\mat{W}_t}\mathcal{W}$. The scalar $s_t > 0$ is the step size. The retraction $\R_{\mat{W}_t}$ is a mapping from the tangent space $T_{\mat{W}_t}\mathcal{W}$ to the Riemannian manifold. Under mild conditions on the retraction $\R$, the classical convergence theory of line-search algorithms in linear spaces generalizes to Riemannian manifolds \citep[see][Chapter 4]{absil08a}.

Observe that the standard (online) learning algorithm for linear regression in $\mathbb{R}^d$,
\begin{equation}\label{eq:update-linreg}
	\vec{w}_{t+1} = \vec{w}_t - s_t (\vec{w}_t^T\vec{x}_t - y_t)\vec{x}_t,
\end{equation}
can be interpreted as a particular case of \eqref{eq:linesearch-manifold} for the linear model $\hat{y}=\vec{w}^T\vec{x}$ in the {\em linear} search space $\mathcal{W}=\mathbb{R}^d$. The Euclidean metric turns $\mathbb{R}^d$ in a (flat) Riemannian manifold. For a scalar function $f:\mathbb{R}^d\rightarrow\mathbb{R}$ of $\vec{w}$, the gradient satisfies 
\begin{equation*}
	Df(\vec{w})[\boldsymbol\delta] = \boldsymbol\delta^T\grad{}{f(\vec{w})},
\end{equation*}
where $Df(\vec{w})[\boldsymbol\delta]$ is the directional derivative of $f$ in the direction $\boldsymbol\delta$, and the natural retraction
\begin{equation*}
	\R_{\vec{w}_t}(-s_t\ \grad{}{f(\vec{w}_t)}) = \vec{w}_t - s_t\ \grad{}{f(\vec{w}_t)},
\end{equation*}
induces a line-search along ``straight lines" which are geodesics (that is paths of shortest length) in linear spaces. With $f(\vec{w})=\frac{1}{2}(\vec{w}^T\vec{x} - y)^2$, one arrives at \eqref{eq:update-linreg}.

This example illustrates that the main ingredients to obtain a concrete algorithm are convenient formulas for the gradient and for the retraction mapping. This paper provides such formulas for three examples of nonlinear matrix search spaces: the Grassmann manifold (Section \ref{sec:grassmann}), the cone of positive definite matrices (Section \ref{sec:cone}), and the set of fixed-rank positive semidefinite matrices (Section \ref{sec:fixed-rank}). Each of those sets will be equipped with {\em quotient Riemannian geometries} that provide convenient formulas for the gradient and for the retractions. Line-search algorithms in quotient Riemannian spaces are discussed in detail in the book of \cite{absil08a}. For the readers convenience, basic concepts and notations are introduced in the next section.

\section{Line-Search Algorithms on Matrix Manifolds}\label{sec:line-search-manifolds}
\begin{figure}[!ht]
	\centering
	\includegraphics[width=250pt]{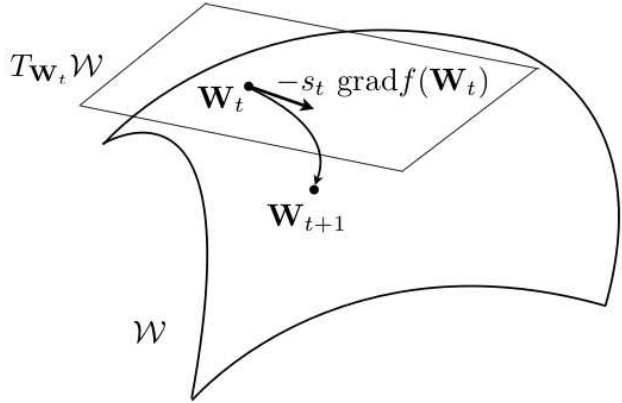}
	\caption{Gradient iteration on a Riemannian manifold. The search direction $-\grad{}{f(\mat{W}_t)}$ belongs to the tangent space $T_{\mat{W}_t}\mathcal{W}$. The updated point $\mat{W}_{t+1}$ automatically remains inside the manifold thanks to the retraction mapping.}
	\label{fig:retraction}
\end{figure}

This section summarizes the exposition of \citet[chap. 3 and 4]{absil08a}.

Restrictions on the search space are generally encoded into optimization algorithms by means of particular constraints or penalties expressed as a function of the search variable. However, when the search space is endowed with a particular manifold structure, it is possible to design an exploration strategy that is consistent with the geometry of the problem and that appropriately turns the problem into an unconstrained optimization problem. This approach is the purpose of optimization algorithms defined on matrix manifolds.

Informally, a manifold $\mathcal{W}$ is a space endowed with a differentiable structure. One usually makes the distinction between embedded submanifolds (subsets of larger manifolds) and quotient manifolds (manifolds described by a set of equivalence classes). An intuitive example of embedded submanifold is the sphere embedded in $\mathbb{R}^d$. A typical example of quotient manifold is the set of $r$-dimensional subspaces in $\mathbb{R}^d$, viewed as a collection of $r$-dimensional orthogonal frames that cannot be superposed by a rotation. The rotational variants of a given frame thus define an equivalence class (denoted using square brackets $[\cdot]$), which is identified as a single point on the quotient manifold.

To develop line-search algorithms, the notion of gradient of a scalar cost function needs to be extended to manifolds. For that purpose, the manifold $\mathcal{W}$ is endowed with a metric $g_{\mat{W}}(\xi_{\mat{W}},\zeta_{\mat{W}})$, which is an inner product defined between elements $\xi_{\mat{W}},\zeta_{\mat{W}}$ of the tangent space $T_{\mat{W}}\mathcal{W}$ at $\mat{W}$. The metric induces a norm on the tangent space $T_{\mat{W}}\mathcal{W}$ at $\mat{W}$: 
\begin{equation*}
\|\xi_{\mat{W}}\|_\mat{W} = \sqrt{g_{\mat{W}}(\xi_{\mat{W}},\xi_{\mat{W}})}.
\end{equation*}
The gradient of a smooth scalar function $f:\mathcal{W}\rightarrow\mathbb{R}$ at $\mat{W}\in\mathcal{W}$ is the only element $\grad{}{f(\mat{W})}\in T_{\mat{W}}\mathcal{W}$ that satisfies
\begin{equation*}\label{eq:gradient-def}
Df(\mat{W})[\boldsymbol\Delta] 
	= g_{\mat{W}}(\boldsymbol\Delta,\grad{}{f(\mat{W})}), \quad \forall \boldsymbol\Delta\in T_{\mat{W}}\mathcal{W},
\end{equation*}
where $\boldsymbol\Delta$ is a matrix representation of a ``geometric" tangent vectors $\xi$, and where
\begin{equation*}
Df(\mat{W})[\boldsymbol\Delta]
  = \lim_{t\rightarrow 0} \frac{f(\mat{W} + t \boldsymbol\Delta) - f(\mat{W})}{t},
\end{equation*}
is the standard directional derivative of $f$ at $\mat{W}$ in the direction $\boldsymbol\Delta$.

For quotient manifolds $\mathcal{W}=\overline{\mathcal{W}}/\sim$, where $\overline{\mathcal{W}}$ is the total space and $\sim$ is the equivalence relation that defines the quotient, the tangent space $T_{[\mat{W}]}\mathcal{W}$ at $[\mat{W}]$ is sufficiently described by the directions that do not induce any displacement in the set of equivalence classes $[\mat{W}]$. This is achieved by restricting the tangent space at $[\mat{W}]$ to horizontal vectors $\bar{\xi}_{\mat{W}}\in T_\mat{W}\overline{\mathcal{W}}$ at $\mat{W}$ that are orthogonal to the equivalence class $[\mat{W}]$. Provided that the metric $\bar{g}_\mat{W}$ in the total space is invariant along the equivalence classes, it defines a metric in the quotient space
\begin{equation*}
	g_{[\mat{W}]}(\xi_{[\mat{W}]},\zeta_{[\mat{W}]}) \triangleq \bar{g}_{\mat{W}}(\bar{\xi}_{\mat{W}},\bar{\zeta}_{\mat{W}}).
\end{equation*}
The horizontal gradient $\overline{\grad{}f(\mat{W})}$ is obtained by projecting the gradient $\grad{}{{f}(\mat{W})}$ in the total space onto the set of horizontal vectors $\bar{\xi}_{\mat{W}}$ at $\mat{W}$.

Natural displacements at $\mat{W}$ in a direction $\xi_{\mat{W}}$ on the manifold are performed by following geodesics (paths of shortest length on the manifold) starting from $\mat{W}$ and tangent to $\xi_{\mat{W}}$. This is performed by means of the exponential mapping
\begin{equation*}\label{eq:exponential-map}
	\mat{W}_{t+1} = \Exp_{\mat{W}_t}(s_t\xi_{\mat{W}_t}),
\end{equation*}
which induces a line-search algorithm along geodesics.

A more general update formula is obtained if we relax the constraint of moving along geodesics. The retraction mapping
\begin{equation*}
	\mat{W}_{t+1} = \R_{\mat{W}_t}(s_t\xi_{\mat{W}_t}),
\end{equation*}
locally approximates the exponential mapping. It provides an attractive alternative to the exponential mapping in the design of optimization algorithms on manifolds, as it reduces the computational complexity of the update while retaining the essential properties that ensure convergence results. When $\xi_{\mat{W}_t}$ coincide with $-\grad{}{f(\mat{W}_t)}$ a gradient descent algorithm on the manifold is obtained. Figure \ref{fig:retraction} pictures a gradient descent update on $\mathcal{W}$.

\section{Linear Regression on the Grassmann Manifold}\label{sec:grassmann}
As a preparatory step to Section \ref{sec:fixed-rank}, we review the online subspace learning \citep{oja92a,crammer06a,warmuth07a} in the present framework. Let $\mathcal{X}=\mathcal{Y}=\mathbb{R}^d$, and consider the linear model
\begin{equation*}
	\hat{\vec{y}} = \mat{U}\mat{U}^T\vec{x},
\end{equation*}
with $\mat{U}\in\Stiefel = \set{\mat{U}\in\mathbb{R}^{d\times r}}{\mat{U}^T\mat{U} = \mat{I}}$, the Stiefel manifold of $r$-dimensional orthonormal bases in $\mathbb{R}^d$. The quadratic loss is then
\begin{equation}\label{eq:subspace-loss}
	f(\mat{U}) = \ell(\hat{\vec{y}},\vec{x}) 
						 = \frac{1}{2}\|\hat{\vec{y}} - \vec{x}\|_2^2
						 = \frac{1}{2}\|\mat{U}\mat{U}^T\vec{x} - \vec{x}\|_2^2.
\end{equation}
Because the cost \eqref{eq:subspace-loss} is invariant by orthogonal transformation $\mat{U}\mapsto \mat{U}\mat{O}$, $\mat{O}\in\OG{r}$, where $\OG{r}=\mathrm{St}(r,r)$ is the orthogonal group, the search space is in fact a set of equivalence classes
\begin{equation*}
	[\mat{U}] = \set{\mat{U}\mat{O}}{\mat{O}\in\OG{r}}.
\end{equation*}
This set is denoted by $\Stiefel/\OG{r}$. It is a {\em quotient representation} of the set of $r$-dimensional subspaces in $\mathbb{R}^d$, that is, the Grassmann manifold $\Grassmann$. The quotient geometries of $\Grassmann$ have been well studied \citep{edelman98a,absil04a}. The metric
\begin{equation*}
	g_{[\mat{U}]}(\xi_{[\mat{U}]},\zeta_{[\mat{U}]}) \triangleq \bar{g}_\mat{U}(\bar{\xi}_\mat{U},\bar{\zeta}_\mat{U}),
\end{equation*}
is induced by the standard metric in $\mathbb{R}^{d\times r}$,
\begin{equation*}
	\bar{g}_{\mat{U}}(\boldsymbol\Delta_1,\boldsymbol\Delta_2) = \trace(\boldsymbol\Delta_1^T\boldsymbol\Delta_2),
\end{equation*}
which is invariant along the fibers, that is, equivalence classes. Tangent vectors $\xi_{[\mat{U}]}$ at $[\mat{U}]$ are represented by horizontal tangent vectors $\bar{\xi}_\mat{U}$ at $\mat{U}$:
\begin{equation*}
	\bar{\xi}_\mat{U} = \Proj_\mat{U} \boldsymbol\Delta = (\mat{I} - \mat{U}\mat{U}^T)  \boldsymbol\Delta,\quad \boldsymbol\Delta\in\mathbb{R}^{d\times r}.
\end{equation*}
Therefore, the gradient admits the simple horizontal representation
\begin{equation}\label{eq:proj-grad}
	\overline{\grad{}{f(\mat{U})}} = \Proj_\mat{U}\ \grad{}{{f}(\mat{U})},
\end{equation}
where $\grad{}{{f}(\mat{U})}$ is defined by the identity
\begin{equation*}
	D{f}(\mat{U})[\boldsymbol\Delta] = \bar{g}_{\mat{U}}(\boldsymbol\Delta,\grad{}{{f}(\mat{U})}).
\end{equation*}

A standard retraction in $\Grassmann$ is the exponential mapping, that induces a line-search along geodesics. The exponential map has the closed-form expression
\begin{equation}\label{eq:geodesic-grassmann}
	\Exp_{\mat{U}}(\bar{\xi}_\mat{U}) = 
		\mat{U}\mat{V}\cos(\boldsymbol\Sigma)\mat{V}^T + \mat{Z}\sin(\boldsymbol\Sigma)\mat{V}^T,
\end{equation}
which is obtained from a singular value decomposition of the horizontal vector $\bar{\xi}_\mat{U} = \mat{Z}\boldsymbol\Sigma \mat{V}^T$. Following \cite{absil04a}, an alternative convenient retraction in $\Grassmann$ is given by
\begin{equation}\label{eq:retraction-grassmann}
	\R_{\mat{U}}(s\bar{\xi}_\mat{U}) = [\mat{U} + s\bar{\xi}_\mat{U}] = \qf(\mat{U} + s\bar{\xi}_\mat{U}),
\end{equation}
where $\qf(\cdot)$ is a function that extracts the orthogonal factor of the QR-decomposition of its argument. A possible advantage of the retraction \eqref{eq:retraction-grassmann} over the retraction \eqref{eq:geodesic-grassmann} is that, in contrast to the SVD computation, the QR decomposition is computed in a fixed number $O(dr^2)$ of arithmetic operations.

With the formulas \eqref{eq:proj-grad} and \eqref{eq:retraction-grassmann} applied to the cost function \eqref{eq:subspace-loss}, the abstract update \eqref{eq:linesearch-manifold} becomes
\begin{equation*}
	\mat{U}_{t+1} = \qf(\mat{U}_t + s_t (\mat{I} - \mat{U}_t\mat{U}_t^T) \vec{x}_t\vec{x}_t^T \mat{U}_t),
\end{equation*}
which is Oja's update for subspace tracking \citep{oja92a}.

\section{Linear Regression on the Cone of Positive Definite Matrices}\label{sec:cone}
The learning of a full-rank positive definite matrix is recast as follows. Let $\mathcal{X}=\mathbb{R}^{d\times d}$ and $\mathcal{Y}=\mathbb{R}$, and consider the model
\begin{equation*}
	\hat{y} = \trace(\mat{W}\mat{X}),
\end{equation*}
with $\mat{W}\in\Cone{d}=\set{\mat{W}\in\mathbb{R}^{d\times d}}{\mat{W} = \mat{W}^T \succ 0}$. Since $\mat{W}$ is symmetric, only the symmetric part of $\mat{X}$ will contribute to the trace. The previous model is thus equivalent to
\begin{equation*}
	\hat{y} = \trace(\mat{W}\Sym(\mat{X})),
\end{equation*}
where $\Sym(\cdot)$ extract the symmetric part of its argument, that is, $\Sym(\mat{B})=(\mat{B}^T+\mat{B})/2$. The quadratic loss is
\begin{equation*}\label{eq:loss-cone-W}
	f(\mat{W}) = \ell(\hat{y},y) = \frac{1}{2} (\trace(\mat{W}\Sym(\mat{X})) - y)^2.
\end{equation*}
The quotient geometries of $\Cone{d}$ are rooted in the matrix factorization
\begin{equation*}
	\mat{W} = \mat{G}\mat{G}^T,\quad \mat{G}\in\GL{d},
\end{equation*}
where $\GL{d}$ is the set of all invertible $d\times d$ matrices. Because the factorization is invariant by rotation, $\mat{G}\mapsto \mat{G}\mat{O}$, $\mat{O}\in\OG{d}$, the search space is once again identified to the quotient
\begin{equation*}
	\Cone{d} \simeq \GL{d}/\OG{d},
\end{equation*}
which represents the set of equivalence classes
\begin{equation*}
	[\mat{G}] = \set{\mat{G}\mat{O}}{\mat{O}\in\OG{d}}.
\end{equation*}
We will equip this quotient with two meaningful Riemannian metrics.

\subsection{A Flat Metric on $\Cone{d}$}\label{eq:flat-cone}
The metric on the quotient $\GL{d}/\OG{d}$:
\begin{equation*}
	g_{[\mat{G}]}(\xi_{[\mat{G}]},\zeta_{[\mat{G}]}) \triangleq \bar{g}_{\mat{G}}(\bar{\xi}_\mat{G},\bar{\zeta}_\mat{G}),
\end{equation*}
is induced by the standard metric in $\mathbb{R}^{d\times d}$,
\begin{equation}
	\bar{g}_{\mat{G}}(\boldsymbol\Delta_1,\boldsymbol\Delta_2) = \trace(\boldsymbol\Delta_1^T\boldsymbol\Delta_2),
\end{equation}
which is invariant by rotation along the set of equivalence classes. As a consequence, it induces a metric $g_{[\mat{G}]}$ on $\Cone{d}$. With this geometry, a tangent vector $\xi_{[\mat{G}]}$ at $[\mat{G}]$ is represented by a horizontal tangent vector $\bar{\xi}_\mat{G}$ at $\mat{G}$ by
\begin{equation*}
	\bar{\xi}_\mat{G} = \Sym(\boldsymbol\Delta)\mat{G},\quad \boldsymbol\Delta\in\mathbb{R}^{d\times d}.
\end{equation*}
The horizontal gradient of
\begin{equation}\label{eq:loss-cone-G}
	f(\mat{G}) = \ell(\hat{y},y) = \frac{1}{2} (\trace(\mat{G}\mat{G}^T\Sym(\mat{X})) - y)^2,
\end{equation}
is the unique horizontal vector $\overline{\grad{}{f(\mat{G})}}$ that satisfies
\begin{equation*}
	Df(\mat{G})[\boldsymbol\Delta] = \bar{g}_{\mat{G}}(\boldsymbol\Delta,\overline{\grad{}{f(\mat{G})}}).
\end{equation*}
Elementary computations yield
\begin{equation*}
	\overline{\grad{}{f(\mat{G})}} = 2(\hat{y} - y) \Sym(\mat{X})\mat{G}.
\end{equation*}
Since the metric is flat, geodesics are straight lines and the exponential mapping is
\begin{equation*}
	\Exp_{\mat{G}}(\bar{\xi}_\mat{G}) = [\mat{G} + \bar{\xi}_\mat{G}] = \mat{G} + \bar{\xi}_\mat{G}.
\end{equation*}
Those formulas applied to the cost \eqref{eq:loss-cone-G} turns the abstract update \eqref{eq:linesearch-manifold} into the simple formula
\begin{equation}\label{eq:online-flat}
	\mat{G}_{t+1} = \mat{G}_t - 2 s_t (\hat{y}_t - y_t)\Sym(\mat{X}_t)\mat{G}_t,
\end{equation}
for an online gradient algorithm and
\begin{equation}\label{eq:batch-flat}
	\mat{G}_{t+1} = \mat{G}_t - 2s_t\frac{1}{n}\sum_{i=1}^n (\hat{y}_i - y_i)\Sym(\mat{X}_i)\mat{G}_t,
\end{equation}
for a batch gradient algorithm.

\subsection{The Affine-Invariant Metric on $\Cone{d}$}\label{sec:affine-invariant-cone}
Because $\Cone{d} \simeq \GL{d}/\OG{d}$ is the quotient of two Lie groups, its (reductive) geometric structure can be further exploited \citep{faraut94a}. Indeed the group $\GL{d}$ has a natural action on $\Cone{d}$ via the transformation $\mat{W}\mapsto\mat{A}\mat{W}\mat{A}^T$ for any $\mat{A}\in\GL{d}$. The affine-invariant metric admits interesting invariance properties to these transformations. To build such an affine-invariant metric, the metric at identity
\begin{equation*}
	g_{\mat{I}}(\xi_{\mat{I}},\zeta_{\mat{I}}) = \trace(\xi_{\mat{I}}\zeta_{\mat{I}}),
\end{equation*}
is extended to the entire space to satisfy the invariance property
\begin{equation*}
	g_{\mat{I}}(\xi_{\mat{I}},\zeta_{\mat{I}})
	= g_\mat{W}(\mat{W}^{\frac{1}{2}}\xi_{\mat{I}}\mat{W}^{\frac{1}{2}},\mat{W}^{\frac{1}{2}}\zeta_{\mat{I}}\mat{W}^{\frac{1}{2}})
	= g_\mat{W}(\xi_{\mat{W}},\zeta_{\mat{W}}).
\end{equation*}
The resulting metric on $\Cone{d}$ is defined by
\begin{equation}\label{eq:affine-invariant-metric-cone}
	g_{\mat{W}}(\xi_{\mat{W}},\zeta_{\mat{W}}) = \trace(\xi_{\mat{W}}\mat{W}^{-1}\zeta_{\mat{W}}\mat{W}^{-1}).
\end{equation}
The affine-invariant geometry of $\Cone{d}$ has been well studied, in particular in the context of information geometry \citep{smith05a}. Indeed, any positive definite matrix $\mat{W}\in\Cone{d}$ can be identified to the multivariate normal distribution of zero mean $\mathcal{N}(0,\mat{W})$, whose probability density is $p(\vec{z};\mat{W}) = \frac{1}{Z}\exp(-\frac{1}{2}\vec{z}^T\mat{W}^{-1}\vec{z})$, where $Z$ is a normalizing constant. Using such a metric allows to endow the space of parameters $\Cone{d}$ with a distance that reflects the proximity of the probability distributions. The Riemannian metric thus distorts the Euclidean distances between positive definite matrices in order to reflect the amount of information between the two associated probability distributions. If $\xi_{\mat{W}}$ is a tangent vector to $\mat{W}\in\Cone{d}$, we have the following approximation for the Kullback-Leibler divergence (up to third order terms)
\begin{equation*}
	D_{KL}(p(\vec{z};\mat{W})||p(\vec{z};\mat{W} + \xi_{\mat{W}}))\approx \frac{1}{2}\; g_{\mat{W}}^{FIM}(\xi_{\mat{W}},\xi_{\mat{W}}) = \frac{1}{2}\; g_\mat{W}(\xi_{\mat{W}},\xi_{\mat{W}}),
\end{equation*}
where $g_{\mat{W}}^{FIM}$ is the well-known Fisher information metric at $\mat{W}$, which coincides with the affine-invariant metric \eqref{eq:affine-invariant-metric-cone} \citep{smith05a}. With this geometry, tangent vectors $\xi_\mat{W}$ are expressed as
\begin{equation*}
	\xi_\mat{W} = \mat{W}^{\frac{1}{2}}\Sym(\boldsymbol\Delta)\mat{W}^{\frac{1}{2}},\quad \boldsymbol\Delta\in\mathbb{R}^{d\times d}.
\end{equation*}
The gradient $\grad{}{f(\mat{W})}$ is given by
\begin{equation*}
	Df(\mat{W})[\boldsymbol\Delta] = g_{\mat{W}}(\boldsymbol\Delta,\grad{}{f(\mat{W})}).
\end{equation*}
Applying this formula to \eqref{eq:loss-cone-W} yields
\begin{equation}\label{eq:gradient-affine-invariant}
	\grad{}{f(\mat{W})} = (\hat{y} - y)\mat{W}\Sym(\mat{X})\mat{W}.
\end{equation}
The exponential mapping has the closed-form expression
\begin{equation}\label{eq:expmap-affine-invariant}
	\Exp_\mat{W}(\xi_\mat{W}) = 
		\mat{W}^{\frac{1}{2}}\exp(\mat{W}^{-\frac{1}{2}}\xi_\mat{W}\mat{W}^{-\frac{1}{2}})\mat{W}^{\frac{1}{2}}.
\end{equation}
Its first-order approximation provides the convenient retraction
\begin{equation}\label{eq:approx-expmap-affine-invariant}
	\R_{\mat{W}}(s\xi_{\mat{W}}) = \mat{W} - s \xi_{\mat{W}}.
\end{equation}
The formulas \eqref{eq:gradient-affine-invariant} and \eqref{eq:expmap-affine-invariant} applied to the cost \eqref{eq:loss-cone-W} turn the abstract update \eqref{eq:linesearch-manifold} into
\begin{equation*}
	\mat{W}_{t+1} =
	 \mat{W}_{t}^{\frac{1}{2}}\exp(-s_t(\hat{y}_t - y_t)
									\mat{W}_{t}^{\frac{1}{2}}\Sym(\mat{X}_{t})\mat{W}_{t}^{\frac{1}{2}})\mat{W}_{t}^{\frac{1}{2}}.
\end{equation*}
With the alternative retraction \eqref{eq:approx-expmap-affine-invariant}, the update becomes
\begin{equation*}
	\mat{W}_{t+1} = \mat{W}_t - s_t(\hat{y}_t - y_t)\mat{W}_{t}\Sym(\mat{X}_{t})\mat{W}_{t},
\end{equation*}
which is the update of \cite{davis07a} based on the LogDet divergence (see Section \ref{sec:closeness}).

\subsection{The Log-Euclidean Metric on $\Cone{d}$}\label{sec:log-euclidean-cone}
For the sake of completeness, we briefly review a third Riemannian geometry of $\Cone{d}$, that exploits the property
\begin{equation*}
	\mat{W} = \exp(\mat{S}),\quad \mat{S} = \mat{S}^T \in\mathbb{R}^{d\times d}.
\end{equation*}
The matrix exponential thus provides a global diffeomorphism between $\Cone{d}$ and the linear space of $d\times d$ symmetric matrices. This geometry is studied in detail in the paper \citep{arsigny06a}. The cost function
\begin{equation*}
	f(\mat{S}) = \ell(\hat{y},y) = \frac{1}{2} (\trace(\exp(\mat{S})\Sym(\mat{X})) - y)^2,
\end{equation*} 
thus defines a cost function in the linear space of symmetric matrices. The gradient of this cost function is given by
\begin{equation*}
	\grad{}{f(\mat{S})} = (\hat{y}_t - y_t) \Sym(\mat{X}_t),
\end{equation*}
and the retraction is 
\begin{equation*}
	\R_{\mat{S}}(s\xi_{\mat{S}}) = \exp(\log\mat{W} + s\xi_{\mat{S}}).
\end{equation*}
The corresponding gradient descent update is
\begin{equation*}
	\mat{W}_{t+1} = \exp(\log\mat{W}_{t} - s_t (\hat{y}_t - y_t) \Sym(\mat{X}_t)),
\end{equation*}
which is the update of \cite{tsuda05a} based on the von Neumann divergence.

\section{Linear Regression on Fixed-Rank Positive Semidefinite Matrices}\label{sec:fixed-rank}
We now present the proposed generalizations to fixed-rank positive semidefinite matrices.

\subsection{Linear Regression with a Flat Geometry}
The generalization of the results of Section \ref{eq:flat-cone} to the set $\FixedRank$ is a straightforward consequence of the factorization
\begin{equation*}
	\mat{W} = \mat{G}\mat{G}^T,\quad \mat{G}\in\mathbb{R}_{*}^{d\times r},
\end{equation*}
where $\mathbb{R}_{*}^{d\times r} = \set{\mat{G}\in\mathbb{R}^{d\times r}}{\det(\mat{G}^T\mat{G}) \neq 0}$. Indeed, the flat quotient geometry of the manifold $\Cone{d}\simeq\GL{d}/\OG{d}$ is generalized to the quotient geometry of $\FixedRank\simeq\mathbb{R}_{*}^{d\times r}/\OG{r}$ by a mere adaptation of matrix dimension, leading to the updates \eqref{eq:online-flat} and \eqref{eq:batch-flat} for matrices $\mat{G}_t\in\mathbb{R}_*^{d\times r}$. The mathematical derivation of these updates is a straight application of the material presented in the paper of \cite{journee09a}, where the quotient geometry of $\FixedRank\simeq\mathbb{R}_{*}^{d\times r}/\OG{r}$ is studied in details. In the next section, we propose an alternative geometry that jointly learns a $r$-dimensional subspace and a full-rank quadratic model in this subspace.

\subsection{Linear Regression with a Polar Geometry}
In contrast to the flat geometry, the affine-invariant geometry of $\Cone{d}\simeq\GL{d}/\OG{d}$ does not generalize directly to $\FixedRank\simeq\mathbb{R}_{*}^{d\times r}/\OG{r}$ because $\mathbb{R}_{*}^{d\times r}$ is not a group. However, a generalization is possible by considering the polar matrix factorization 
\begin{equation*}
	\mat{G} = \mat{U}\mat{R},\quad \mat{U}\in\Stiefel,\ \mat{R}\in\Cone{r}.
\end{equation*}
It is obtained from the singular value decomposition of $\mat{G}=\mat{Z}\boldsymbol\Sigma\mat{V}^T$ as $\mat{U}=\mat{Z}\mat{V}^T$ and $\mat{R}=\mat{V}\boldsymbol\Sigma\mat{V}^T$ \citep{golub06a}. This gives a polar parametrization of $\FixedRank$
\begin{equation*}
	\mat{W} = \mat{U}\mat{R}^2\mat{U}^T.
\end{equation*}
This development leads to the quotient representation
\begin{equation}\label{eq:polar-quotient}
	\FixedRank \simeq (\Stiefel\times\Cone{r})/\OG{r},
\end{equation}
based on the invariance of $\mat{W}$ to the transformation $(\mat{U},\mat{R}^2)\mapsto(\mat{U}\mat{O},\mat{O}^T\mat{R}^2\mat{O})$, $\mat{O}\in\OG{r}$. It thus describes the set of equivalence classes
\begin{equation*}
	[(\mat{U},\mat{R}^2)] = \set{(\mat{U}\mat{O},\mat{O}^T\mat{R}^2\mat{O})}{\mat{O}\in\OG{r}}.
\end{equation*}
The cost function is now given by
\begin{equation}\label{eq:loss-fixed-rank-polar}
	f(\mat{U},\mat{R}^2) = \ell(\hat{y},y) = \frac{1}{2} (\trace(\mat{U}\mat{R}^2\mat{U}^T\Sym(\mat{X})) - y)^2.
\end{equation}
The Riemannian geometry of \eqref{eq:polar-quotient} has been recently studied by \cite{bonnabel09a}. A tangent vector $\xi_\mat{[W]}=(\xi_{\mat{U}},\xi_{\mat{R}^2})_{[\mat{U},\mat{R}^2]}$ at $[(\mat{U},\mat{R}^2)]$ is described by a horizontal tangent vector $\bar{\xi}_\mat{W}=(\bar{\xi}_{\mat{U}},\bar{\xi}_{\mat{R}^2})_{(\mat{U},\mat{R}^2)}$ at $(\mat{U},\mat{R}^2)$ by
\begin{equation*}
	\bar{\xi}_{\mat{U}} = \Proj_{\mat{U}} \boldsymbol\Delta,\ \boldsymbol\Delta\in\mathbb{R}^{d\times r},
		\qquad \bar{\xi}_{\mat{R}^2} = \mat{R}\Sym(\boldsymbol\Psi)\mat{R},\ \boldsymbol\Psi\in\mathbb{R}^{r\times r}.
\end{equation*}
The metric
\begin{eqnarray}\label{eq:metric-polar}
	g_{[\mat{W}]}(\xi_{[\mat{W}]},\zeta_{[\mat{W}]}) &\triangleq& \bar{g}_{\mat{W}}(\bar{\xi}_{\mat{W}},\bar{\zeta}_{\mat{W}})\nonumber\\
		&=& \frac{1}{\lambda}\ \bar{g}_{\mat{U}}(\bar{\xi}_{\mat{U}},\bar{\zeta}_{\mat{U}}) 
						+ \frac{1}{1-\lambda}\ \bar{g}_{\mat{R}^2}(\bar{\xi}_{\mat{R}^2},\bar{\zeta}_{\mat{R}^2}),
\end{eqnarray}
where $\lambda\in (0,1)$, is induced by the metric of $\Stiefel$ and the affine-invariant metric of $\Cone{r}$,
\begin{equation*}
	\bar{g}_{\mat{U}}(\boldsymbol\Delta_1,\boldsymbol\Delta_2) = \trace(\boldsymbol\Delta_1^T\boldsymbol\Delta_2),\qquad
	\bar{g}_{\mat{R}^2}(\boldsymbol\Psi_1,\boldsymbol\Psi_2)   = \trace(\boldsymbol\Psi_1\mat{R}^{-2}\boldsymbol\Psi_2\mat{R}^{-2}).
\end{equation*}
The proposed metric is invariant along the set of equivalence classes and thus induces a quotient structure on $\FixedRank$. Alternative metrics on $\Cone{r}$ can be considered as long as the metric remains invariant along the set of equivalence classes. For instance, the log-Euclidean metric discussed in Section \ref{sec:log-euclidean-cone} would qualify as a valid alternative.

A retraction is provided by distinct retractions on $\mat{U}$ and $\mat{R}^2$,
\begin{eqnarray}
	\label{eq:retraction-U}
	\R_{\mat{U}}(s\bar{\xi}_{\mat{U}})     &=& \qf(\mat{U} + s\bar{\xi}_\mat{U})\\
	\label{eq:retraction-R}
	\R_{\mat{R}^2}(s\bar{\xi}_{\mat{R}^2}) &=& \mat{R}\exp(s\mat{R}^{-1}\bar{\xi}_{\mat{R}^2}\mat{R}^{-1})\mat{R}.
\end{eqnarray}
One should observe that this retraction is not the exponential mapping of $\FixedRank$. This illustrates the interest of considering more general retractions than the exponential mapping. Indeed, as discussed in the paper of \cite{bonnabel09a}, the geodesics (and therefore the exponential mapping) do not appear to have a closed form in the considered geometry. Combining the gradient of \eqref{eq:loss-fixed-rank-polar} with the retractions \eqref{eq:retraction-U} and \eqref{eq:retraction-R} gives
\begin{equation*}
	\begin{split}
			\mat{U}_{t+1} &= 
			\qf\left(\mat{U}_t - 2\lambda s_t (\hat{y}_t - y_t) (\mat{I}-\mat{U}_t\mat{U}_t^T) \Sym(\mat{X}_t)\mat{U}_t \mat{R}^2_t\right),\\
		\mat{R}^2_{t+1} &= 
						\mat{R}_t 
	             \exp\left(-(1-\lambda) s_t (\hat{y}_t - y_t)
	                        	\mat{R}_t \mat{U}_t^T \Sym(\mat{X}_t) \mat{U}_t\mat{R}_t\right)
							\mat{R}_t.
	\end{split}
\end{equation*}
A factorization $\mat{R}_{t+1}\mat{R}_{t+1}^T$ of $\mat{R}^2_{t+1}$ is obtained thanks to the property of matrix exponential, $\exp(\mat{A})^{\frac{1}{2}}=\exp(\frac{1}{2}\mat{A})$. Updating $\mat{R}_{t+1}$ instead of $\mat{R}_{t+1}^2$ is thus more efficient from a computational point of view, since it avoids the computation of a square root a each iteration. 
This yields the online gradient descent algorithm
\begin{equation}\label{eq:online-polar}
	\begin{split}
			\mat{U}_{t+1} &= 
			\qf\left(\mat{U}_t - 2\lambda s_t (\hat{y}_t - y_t) (\mat{I}-\mat{U}_t\mat{U}_t^T) \Sym(\mat{X}_t)\mat{U}_t \mat{R}^2_t\right),\\
			\mat{R}_{t+1} &= 
						\mat{R}_t 
	             \exp\left(-\frac{1}{2}(1-\lambda) s_t (\hat{y}_t - y_t)
	                        	\mat{R}_t \mat{U}_t^T \Sym(\mat{X}_t) \mat{U}_t\mat{R}_t\right),
	\end{split}
\end{equation}
and the batch gradient descent algorithm
\begin{equation}\label{eq:batch-polar}
	\begin{split}
		\mat{U}_{t+1}	&= 
			\qf\left(\mat{U}_t 
				- 2\lambda s_t\frac{1}{n}\sum_{i=1}^n(\hat{y}_i - y_i) (\mat{I}-\mat{U}_t\mat{U}_t^T) \Sym(\mat{X}_i)\mat{U}_t \mat{R}^2_t\right),\\
			\mat{R}_{t+1} &= 
							\mat{R}_t 
              	\exp\left(-\frac{1}{2}(1-\lambda) s_t\frac{1}{n}\sum_{i=1}^n (\hat{y}_i - y_i)
                        	\mat{R}_t \mat{U}_t^T\Sym(\mat{X}_i) \mat{U}_t \mat{R}_t\right).
	\end{split}
\end{equation}

\section{Algorithms}\label{sec:algorithms}
This section documents implementation details of the proposed algorithms. Generic pseudo-codes are provided in Figure \ref{fig:meta-algos} and Table \ref{tab:computational-complexities} summarizes computational complexities.

\begin{figure}
	\footnotesize
	\centering
	\begin{tabular}{cc}
	\toprule
	\begin{tabular}{c|c}
		\multicolumn{2}{c}{\vspace{-4mm}}\\
		\begin{minipage}[c]{0.45\textwidth}
			\begin{center}
	 			Batch regression
			\end{center}
		\end{minipage}
		& 
		\begin{minipage}[c]{0.45\textwidth}
			\begin{center}
				Online regression
			\end{center}
		\end{minipage}\\
		\multicolumn{2}{c}{\vspace{-4mm}}
	\end{tabular}\\
	\midrule
	\begin{tabular}{c|c}
		\multicolumn{2}{c}{\vspace{-2mm}}\\
		\begin{minipage}[c]{0.45\textwidth}
			\begin{algorithmic}[1]
				\ENSURE  $\{(\mat{X}_i,y_i)\}_{i=1}^n$
				\REQUIRE $\mat{G}_0$ or $(\mat{U}_0,\mat{R}_0)$, $\lambda$
				\STATE $t = 0$
				\REPEAT
				\STATE
				\STATE
				\STATE
				\STATE
				\STATE Compute Armijo step $s_A$ from \eqref{eq:armijo}
				\STATE Perform update \eqref{eq:batch-flat} or \eqref{eq:batch-polar} using $s_A$
				\STATE
				\STATE
				\STATE $t = t + 1$
				\UNTIL{stopping criterion \eqref{eq:stop-criterion} is satisfied}
				\RETURN $\mat{G}_t$
			\end{algorithmic}
		\end{minipage}
		& 
		\begin{minipage}[c]{0.45\textwidth}
			\begin{algorithmic}[1]
				\ENSURE  $\{(\mat{X}_t,y_t)\}_{t\geq 1}$
				\REQUIRE $\mat{G}_0$ or $(\mat{U}_0,\mat{R}_0)$, $\lambda$, $p$, $s$, $t_0$, $T$
				\STATE $t = 0, \mathtt{count} = p$
				\WHILE{$t \leq T$}
				\IF{$\mathtt{count} > 0$}
				\STATE Accumulate gradient
				\STATE $\mathtt{count} = \mathtt{count} - 1$
				\ELSE
				\STATE Compute step size $s_t$ according to \eqref{eq:step}
				\STATE Perform update \eqref{eq:online-flat} or \eqref{eq:online-polar} using $s_t$
				\STATE $\mathtt{count} = p$
				\ENDIF
				\STATE $t = t + 1$
				\ENDWHILE
				\RETURN $\mat{G}_T$
			\end{algorithmic}
			\end{minipage}\\
			\multicolumn{2}{c}{\vspace{-2mm}}
		\end{tabular}\\
	\bottomrule
	\end{tabular}
	\caption{Pseudo-codes for the proposed batch and online algorithms.}
	\label{fig:meta-algos}
\end{figure}

\begin{table}[!ht]
	\centering
	\small
	\begin{tabular}{llllll}
		\toprule
		Data type & Input space & Batch flat \eqref{eq:batch-flat} & Batch polar \eqref{eq:batch-polar} & Online flat \eqref{eq:online-flat} & Online polar \eqref{eq:online-polar}\\
		\midrule
		$\mat{X}$          & $\mathbb{R}^{d\times d}$ & $O(d^2rn)$ & $O(d^2r^2n)$ & $O(d^2rp)$ & $O(d^2r^2p)$\\
		$\vec{x}\vec{x}^T$ & $\mathbb{R}^{d}$         & $O(drn)$   & $O(dr^2n)$   & $O(drp)$   & $O(dr^2p)$\\
		\bottomrule
	\end{tabular}
	\caption{Computational costs of the proposed algorithms.}
	\label{tab:computational-complexities}
\end{table}
\subsection{From Subspace Learning to Distance Learning}
The update expressions \eqref{eq:batch-polar} and \eqref{eq:online-polar} show that $\lambda$, the tuning parameter of the Riemannian metric \eqref{eq:metric-polar}, acts as a weighting factor on the search direction. A proper tuning of this parameter allows us to place more emphasis either on the learning of the subspace $\mat{U}$ or on the distance in that subspace $\mat{R}^2$. In the case $\lambda=1$, the algorithm only performs subspace learning. Conversely, in the case $\lambda=0$, the algorithm learns a distance for a fixed range space (see Section \ref{sec:closeness}). Intermediate values of $\lambda$ continuously interpolate between the subspace learning problem and the distance learning problem at fixed range space.

A proper tuning of $\lambda$ is of interest when a good estimate of the subspace is available (for instance a subspace given by a proper dimension reduction technique) or when too few observations are available to jointly estimate the subspace and the distance within that subspace. In the latter case, one has the choice to favor either subspace or distance learning. 

Experimental results of Section \ref{sec:experiments} recommend the value $\lambda=0.5$ as the default setting. 

\subsection{Invariance Properties}
A nice property of the proposed algorithms is their invariance with respect to rotations $\mat{W}\mapsto \mat{O}^T \mat{W} \mat{O}$, $\forall \mat{O}\in\OG{d}$. This invariance comes from the fact that the chosen metrics are invariant to rotations.  A practical consequence is that a rotation of the input matrix $\mat{X}\mapsto \mat{O}\mat{X}\mat{O}^T$ (for instance a whitening transformation of the vectors $\vec{x}\mapsto \mat{O}\vec{x}$ if $\mat{X}=\vec{x}\vec{x}^T$) will not affect the behavior of the algorithms.

Besides being invariant to rotations, algorithms \eqref{eq:online-polar} and \eqref{eq:batch-polar} are invariant with respect to scalings $\mat{W}\mapsto\mu \mat{W}$ with $\mu > 0$. Consequently, a scaling of the input data $(\mat{X},y) \mapsto (\mu \mat{X}, \mu y)$, such as a change of units, will not affect the behavior of these algorithms.

\subsection{Mini-Batch Extension of Online Algorithms}
We consider a mini-batch extension of stochastic gradient algorithms. It consists in performing each gradient step with respect to $p\geq 1$ examples at a time instead of a single one. This is a classical speedup and stabilization heuristic for stochastic gradient algorithms. In the particular case $p=1$, one recovers plain stochastic gradient descent. Given $p$ samples $(\mat{X}_{t,1},y_{t,1}),...,(\mat{X}_{t,p},y_{t,p})$, received at time $t$, the abstract update \eqref{eq:linesearch-manifold} becomes
\begin{equation*}
	\mat{W}_{t+1} = \R_{\mat{W}_t}\left(-s_t\ \frac{1}{p}\sum_{i=1}^{p}\grad{}{\ell(\hat{y}_{t,i},y_{t,i})}\right).
\end{equation*}

\subsection{Strategies for Choosing the Step Size}
We here present strategies for choosing the step size in both the batch and online cases.
\subsubsection{Batch Algorithms}
For batch algorithms, classical backtracking methods exist \citep[see][]{nocedal06a}. In this paper, we use the Armijo step $s_A$ defined at each iteration by the condition
\begin{equation}\label{eq:armijo}
f(\R_{\mat{W}_t}(-s_A\ \grad{}{f(\mat{W}_t)})) 
  \leq f(\mat{W}_t) + c \|\grad{}{f(\mat{W}_t)}\|^2_{\mat{W}_t},
\end{equation}
where $\mat{W}_t\in S_+(r,d)$ is the current iterate, $c\in(0,1)$, $f$ is the empirical cost \eqref{eq:empirical-cost} and $\R_{\mat{W}}$ is the chosen retraction. In this paper, we choose the particular value $c=0.5$ and repetitively divide by $2$ a specified maximum step size $s_{max}$ until condition \eqref{eq:armijo} is satisfied for the considered iteration. In order to reduce the dependence on $s_{max}$ in a particular problem, it is chosen inversely proportional to the norm of the gradient at each iteration,
\begin{equation*}
	s_{max} = \frac{s_{0}}{\|\grad{}{f(\mat{W}_t)}\|_{\mat{W}_t}}.
\end{equation*}
A typical value of $s_0=100$ showed satisfactory results for all the considered problems.  

\subsubsection{Online Algorithms}
For online algorithms, the choice of the step size is more involved. In this paper, 
the step size schedule $s_t$ is chosen as
\begin{equation}\label{eq:step}
	s_t = \frac{s}{\hat{\mu}_{grad}}\times\frac{nt_0}{nt_0 + t},
\end{equation}
where $s > 0$, $n$ is the number of considered learning samples, $\hat{\mu}_{grad}$ is an estimate of the average gradient norm $\|\grad{}{f(\mat{W}_0)}\|_{\mat{W}_0}$, and $t_0 > 0$ controls the annealing rate of $s_t$. During a pre-training phase of our online algorithms, we select a small subset of learning samples and try the values $2^k$ with $k=-3,...,3$ for both $s$ and $t_0$. The values of $s$ and $t_0$ that provide the best decay of the cost function are selected to process the complete set of learning samples.

\subsection{Stopping Criterion}
Batch algorithms are stopped when the value or the relative change of the empirical cost $f$ is small enough, or when the relative change in the parameter variation is small enough,
\begin{equation}\label{eq:stop-criterion}
	f(\mat{W}_{t+1})  \leq \epsilon_{tol},
		\quad \text{or} \quad
	\frac{f(\mat{W}_{t+1}) - f(\mat{W}_{t})}{f(\mat{W}_{t})} \leq \epsilon_{tol},
		\quad \text{or} \quad
	\frac{\|\mat{G}_{t+1} - \mat{G}_{t}\|_F}{\|\mat{G}_{t}\|_F} \leq \epsilon_{tol}.
\end{equation}
We found $\epsilon_{tol}=10^{-5}$ to be a good trade-off between accuracy and convergence time. 

Online algorithms are run for a fixed number of epochs (number of passes through the set of learning samples). Typically, a few epochs are sufficient to attain satisfactory results.

\subsection{Convergence}\label{sec:convergence}
Gradient descent algorithms on matrix manifolds share the well-characterized convergence properties of their analog in $\mathbb{R}^d$. Batch algorithms converge linearly to a local minimum of the empirical cost that depends on the initial condition. Online algorithms converge asymptotically to a local minimum of the expected loss. They intrinsically have a much slower convergence rate than batch algorithms, but they generally decrease faster the expected loss in the large-scale regime~\citep{bottou08a}. The main idea is that, given a training set of samples, an inaccurate solution may indeed have the same or a lower expected cost than a well-optimized one.

When learning a matrix $\mat{W}\in\Cone{d}$, the problem is convex and the proposed algorithms converge toward a global minimum of the cost function, regardless of the initial condition. When learning a low-rank matrix $\mat{W}\in \FixedRank$, with $r < d$, the proposed algorithms converge to a local minimum of the cost function. This is not the case for heuristic methods proposed in the literature, which first reduce the dimensionality of the data before fitting a full-rank model on the reduced data \citep{davis08a,weinberger09a}.

For batch algorithms, the local convergence results follow from the convergence theory of line-search algorithms on Riemannian manifolds \citep[see, for example,][]{absil08a}.

For online algorithms, one can prove that the algorithm based on the flat geometry enjoys almost sure asymptotic convergence to a local minimum of the expected cost. In that case, the parameter $\mat{G}$ belongs to an Euclidean space and the convergence results presented by \cite{bottou98a} apply (see Appendix \ref{sec:proof-sgd} for the main ideas of the proof). 

In contrast, when the polar parametrization is used, the convergence results presented by \cite{bottou98a} do not apply directly because of the quotient nature of the search space. Because the extension would require technical arguments beyond the scope of the present paper, we refrain from stating a formal convergence result for the online algorithm based on the polar geometry, even though the result is quite plausible.

Due to the nonconvex nature of the considered rank-constrained problems, the convergence results are only local and little can be presently said about the global convergence of the algorithms. A global analysis of the critical points of the cost functions studied in the present paper is nevertheless not hopeless and could be facilitated by the considered low-rank parametrizations. For instance, global convergence properties have been established for PCA algorithms from an explicit analysis of the critical points \citep{chen98a}. Also, recent results suggest good global convergence properties for closely related rank minimization problems \citep{recht10a}. Experimental results suggest the same conclusions for the algorithms considered in this paper, which means that further research on global convergence results is certainly deserved.

\section{Discussion}\label{sec:discussion}
This section presents connections with existing works and extensions of the regression model.

\subsection{Closeness-Based Approaches}\label{sec:closeness}
A standard derivation of learning algorithms is as follows \citep{kivinen97a}. The (online) update at time $t$ is viewed as an (approximate) solution of
\begin{equation}\label{eq:closeness-interpretation}
	\mat{W}_{t+1} = \argmin_{\mat{W}\in \mathcal{W}}\ D(\mat{W},\mat{W}_t)\ +\ s_t\ \ell(\hat{y},y_t),
\end{equation}
where $D$ is a well-chosen measure of closeness between elements of $\mathcal{W}$ and $s_t$ is a trade-off parameter that controls the balance between the conservative term $D(\mat{W},\mat{W}_t)$ and the innovation (or data fitting) term $\ell(\hat{y},y_t)$. One solves \eqref{eq:closeness-interpretation} by solving the algebraic equation
\begin{equation}\label{eq:closeness-optim-condition}
	\grad{}{\ D(\mat{W},\mat{W}_t)} = - s_t\ \grad{}{\ \ell(\hat{y}_{t+1},y_t)},
\end{equation}
which is a first-order (necessary) optimality condition. If the search space $\mathcal{W}$ is a Riemannian manifold and if the closeness measure $D(\mat{W},\mat{W}_t)$ is the Riemannian distance, the solution of \eqref{eq:closeness-optim-condition} is
\begin{equation*}
	\mat{W}_{t+1} = \Exp_{\mat{W}_t}(-s_t\ \grad{}{\ \ell(\hat{y}_{t+1},y_t)}) .
\end{equation*}
Because $\hat{y}_{t+1}$ must be evaluated in $\mat{W}_{t+1}$, this update equation is implicit. However,  $\hat{y}_{t+1}$ is generally replaced by $\hat{y}_t$ (which is equal to $\hat{y}_{t+1}$ up to first order terms in $s_t$), which gives the update \eqref{eq:linesearch-manifold} where the exponential mapping is chosen as a retraction.

Bregman divergences have been popular closeness measures for $D(\mat{W},\mat{W}_t)$ because they render the optimization of \eqref{eq:closeness-interpretation} convex. Bregman divergences on the cone of positive definite matrices include the von Neumann divergence
\begin{equation*}
D_{vN}(\mat{W},\mat{W}_t) = \trace(\mat{W} \log \mat{W} - \mat{W} \log \mat{W}_t - \mat{W} + \mat{W}_t),
\end{equation*} 
and the LogDet divergence
\begin{equation*}
D_{ld}(\mat{W},\mat{W}_t) = \trace(\mat{W}\mat{W}_t^{-1}) - \log\det(\mat{W}\mat{W}_t^{-1}) - d.
\end{equation*}
We have shown in Section \ref{sec:cone} that the resulting updates can be interpreted as line-search updates for the log-Euclidean metric and the affine-invariant metric of $\Cone{d}$ and for specific choices of the retraction mapping.

Likewise, the algorithm \eqref{eq:online-flat} can be recast in the framework \eqref{eq:closeness-interpretation} by considering the closeness
\begin{equation*}
	D_{flat}(\mat{W},\mat{W}_t) = \|\mat{G} - \mat{G}_t\|_F^2,
\end{equation*}
where $\mat{W}=\mat{G}\mat{G}^T$ and $\mat{W}_t=\mat{G}_t\mat{G}_t^T$. Algorithm \eqref{eq:online-polar} can be recast in the framework \eqref{eq:closeness-interpretation} by considering the closeness
\begin{equation*}
	D_{pol}(\mat{W},\mat{W}_t) = 
		{\lambda} \ \sum_{i=1}^{r} \theta_i^2\ +\ 
			(1-\lambda) \ \|\log \mat{R}_t^{-1} \mat{R}^2 \mat{R}_t^{-1}\|_F^2.
\end{equation*}
where the $\theta_i$'s are the principal angles between the subspaces spanned by $\mat{W}$ and $\mat{W}_t$ \citep{golub06a}, and the second term is the affine-invariant distance of $\Cone{d}$ between matrices $\mat{R}^2$ and $\mat{R}^2_t$ involved in the polar representation of $\mat{W}$ and $\mat{W}_t$.

Obviously, these closeness measures are no longer convex due to the rank constraint. However they reduce to the popular divergences in the full-rank case, up to second order terms. In particular, when $\lambda = 1$, the subspace is fixed and one recovers the setup of learning low-rank matrices of a fixed range space \citep{kulis09a}. Thus, the algorithms introduced in the present paper can be viewed as generalizations of the ones presented in the paper of \cite{kulis09a}, where the issue of adapting the range space is presented as an open research question. Each of the proposed algorithms provides an efficient workaround for this problem at the expense of the (potential) introduction of local minima.

\subsection{Handling Inequalities}\label{sec:inequalities}
Inequalities $\hat{y} \leq y$ or $\hat{y} \geq y$ can be considered by treating them as equalities when they are not satisfied. This is equivalent to the minimization of the continuously differentiable cost function
\begin{equation*}
f(\mat{W}) = \ell(\hat{y},y) = \frac{1}{2}\max(0,\rho(\hat{y} - y))^2,
\end{equation*} 
where $\rho=+1$ if $\hat{y}\leq y$ is required and $\rho=-1$ if $\hat{y}\geq y$ is required.

\subsection{Kernelizing the Regression Model}
In this paper, we have not considered the kernelized model 
\begin{equation*}
	\hat{y}=\trace(\mat{W}\phi(\vec{x})\phi(\vec{x})^T),
\end{equation*}
whose predictions can be extended to new input data $\phi(\vec{x})$ in the feature space $\mathcal{F}$ induced by the nonlinear mapping $\phi:\vec{x}\in\mathcal{X}\mapsto\phi(\vec{x})\in\mathcal{F}$. This is potentially a useful extension of the regression model that could be investigated in the light of recent theoretical results in this area \citep[for example][]{chatpatanasiri10a,jain10a}.

\subsection{Connection with Multidimensional Scaling Algorithms}
Given a set of $m$ dissimilarity measures $\mathcal{D}=\{\delta_{ij}\}^m$ between $n$ data objects, multidimensional scaling algorithms search for a $r$-dimensional embedding of the data objects into an Euclidean space representation $\mat{G}\in\mathbb{R}^{n\times r}$~\citep{cox01a,borg05a}. Each row $\vec{g}$ of $\mat{G}$ is the coordinates of a data object in a Euclidean space of dimension $r$. 

Multidimensional scaling algorithms based on gradient descent are equivalent to algorithms~\eqref{eq:online-flat} and~\eqref{eq:batch-flat} when $\mat{X}=(\vec{e}_i-\vec{e}_j)(\vec{e}_i-\vec{e}_j)^T$, where $\vec{e}_i$ is the $i$-th unit vector (see Section~\ref{sec:kernel-learning}), and when the multidimensional scaling reduction criterion is the SSTRESS
\begin{equation*}
SSTRESS(\mat{G}) = \sum_{(i,j)\in\mathcal{D}} (\|\vec{g}_{i} - \vec{g}_{j}\|^2_2 - \delta_{ij})^2.
\end{equation*}
Vectors $\vec{g}_{i}$ and $\vec{g}_{j}$ are the $i$-th and $j$-th rows of matrix $\mat{G}$. Gradient descent is a popular technique in the context of multidimensional scaling algorithms. A stochastic gradient descent approach for minimizing the SSTRESS has also been proposed by \cite{matsuda01a}. A potential area of future work is the application of the proposed online algorithm~\eqref{eq:online-flat} for adapting a batch solution to slight modifications of the dissimilarities over time. This approach has a much smaller computational cost than recomputing the offline solution at every time step. It further allows to keep the coordinate representation coherent over time since the solution do not brutally jumps from a local minimum to another.

\section{Applications}\label{sec:applications}
The choice of an appropriate distance measure is a central issue for many distance-based classification and clustering algorithms such as nearest neighbor classifiers, support vector machines or k-means. Because this choice is highly problem-dependent, numerous methods have been proposed to learn a distance function directly from data. In this section, we present two important distance learning applications that are compatible with the considered regression model and review some relevant literature on the subject.

\subsection{Kernel Learning}\label{sec:kernel-learning}
In kernel-based methods \citep{shawe-taylor04a}, the data samples $\vec{x}_1,...,\vec{x}_n$ are first transformed by a nonlinear mapping $\phi:\vec{x}\in\mathcal{X}\mapsto\phi(\vec{x})\in\mathcal{F}$, where $\mathcal{F}$ is a new feature space that is expected to facilitate pattern detection into the data. 

The kernel function is then defined as the dot product between any two samples in $\mathcal{F}$,
\begin{equation*}\label{eq:kernel_function}
	\kappa(\vec{x}_i,\vec{x}_j) = \phi(\vec{x}_i) \cdot \phi(\vec{x}_j).
\end{equation*}
In practice, the kernel function is represented by a positive semidefinite matrix $\mat{K}\in\mathbb{R}^{n\times n}$ whose entries are defined as $\mat{K}_{ij}=\phi(\vec{x}_i) \cdot \phi(\vec{x}_j)$. This inner product information is used solely to compute the relevant quantities needed by the algorithms based on the kernel. For instance, a distance is implicitly defined by any kernel function as the Euclidean distance between the samples in the new feature space
\begin{equation*}
d_\phi(\vec{x}_i,\vec{x}_j) = \|\phi(\vec{x}_i)-\phi(\vec{x}_j)\|^2=\kappa(\vec{x}_i,\vec{x}_i)+\kappa(\vec{x}_j,\vec{x}_j)-2 \kappa(\vec{x}_i,\vec{x}_j),
\end{equation*}
which can be evaluated using only the elements of the kernel matrix by the formula
\begin{equation*}
d_\phi(\vec{x}_i,\vec{x}_j) = \mat{K}_{ii} + \mat{K}_{jj} - 2 \mat{K}_{ij} 
                            = \trace\left(\mat{K}(\vec{e}_i-\vec{e}_j)(\vec{e}_i-\vec{e}_j)^T\right),
\end{equation*}
which fits into the considered regression model. 

Learning a kernel consists in computing the kernel (or Gram) matrix from scratch or improving a existing kernel matrix based on side-information (in a semi-supervised setting for instance). Data samples and class labels are generally exploited by means of equality or inequality constraints involving pairwise distances or inner products.

Most of the numerous kernel learning algorithms that have been proposed work in the so-called transductive setting, that is, it is not possible to generalize the learned kernel function to new data samples \citep{kwok03a,lanckriet04a,tsuda05a,zhuang09a,kulis09a}. In that setting, the total number of considered samples is known in advance and determines the size of the learned matrix. Recently, algorithms have been proposed to learn a kernel function that can be extended to new points \citep{chatpatanasiri10a,jain10a}. In this paper, we only consider the kernel learning problem in the transductive setting.

When low-rank matrices are considered, kernel learning algorithms can be regarded as dimensionality reduction methods.  Very popular unsupervised algorithms in that context are kernel principal component analysis~\citep{scholkopf98a} and multidimensional scaling \citep{cox01a,borg05a}. Other kernel learning techniques include the maximum variance unfolding algorithm~\citep{weinberger04a} and its semi-supervised version \citep{song08a}, and the kernel spectral regression framework \citep{cai07a} which encompasses many reduction criterion (for example, linear discriminant analysis (LDA), locality preserving projection (LPP), neighborhood preserving embedding (NPE)). See the survey of \cite{yang06a} for a more complete state-of-the-art in this area. 

Since our algorithms are able to compute a low-rank kernel matrix from data, they can be used for unsupervised or semi-supervised dimensionality reduction, depending whether or not the class labels are exploited through the imposed constraints. 

\subsection{Mahalanobis Distance Learning}

Mahalanobis distances generalize the usual Euclidean distance as it allows to transform the data with an arbitrary rotation and scaling before computing the distance. Let $\vec{x}_i,\vec{x}_j\in\mathbb{R}^d$ be two data samples, the (squared) Mahalanobis distance between these two samples is parametrized by a positive definite matrix $\mat{A}\in\mathbb{R}^{d\times d}$ and writes as
\begin{equation}\label{eq:mahalanobis}
d_{\mat{A}}(\vec{x}_i,\vec{x}_j) = (\vec{x}_i - \vec{x}_j)^T \mat{A}\ (\vec{x}_i - \vec{x}_j).
\end{equation}
In the particular case of $\mat{A}$ being equal to the identity matrix, the standard Euclidean distance is obtained. A frequently used matrix is $\mat{A}=\boldsymbol\Sigma^{-1}$, the inverse of the sample covariance matrix. For centered data features, computing this Mahalanobis distance is equivalent to perform a whitening of the data before computing the Euclidean distance. 

For low-rank Mahalanobis matrices, computing the distance is equivalent to first perform a linear data reduction step before computing the Euclidean distance on the reduced data.\footnote{In the low-rank case, one should rigorously refer to \eqref{eq:mahalanobis} as a pseudo-distance. Indeed, one has $d_{\mat{A}}(\vec{x}_i,\vec{x}_j)=0$ with $\vec{x}_i\neq\vec{x}_j$ whenever $(\vec{x}_i-\vec{x}_j)$ lies in the null space of $\mat{A}$.} Learning a low-rank Mahalanobis matrix can thus be seen as learning a linear projector that is used for dimension reduction.

In contrast to kernel functions, Mahalanobis distances easily generalize to new data samples since the sole knowledge of $\mat{A}$ determines the distance function.

In recent years, Mahalanobis distance learning algorithms have been the subject of many contributions that cannot be all enumerated here. We review a few of them, most relevant for the present paper. The first proposed methods have been based on successive projections onto a set of large margin constraints~\citep{xing02a,shalev-shwartz04a}. The method proposed by~\cite{globerson05a} seeks a Mahalanobis matrix that maximizes the between classes distance while forcing to zero the within classes distance. A simpler objective is pursued by the algorithms that optimize the Mahalanobis distance for the specific $k$-nearest neighbor classifier \citep{goldberger05a,torresani07a,weinberger09a}. Bregman projection based methods minimize a particular Bregman divergence under distance constraints. Both batch~\citep{davis07a} and online~\citep{jain08a} formulations have been proposed for learning full-rank matrices. Low-rank matrices have also been considered with Bregman divergences but only when the range space of the matrix is fixed in the first place~\citep{davis08a,kulis09a}. 

\section{Experiments}\label{sec:experiments}
\begin{table}[!ht]
	\small
	\begin{tabular}{lrrrl}
		\toprule
		Data Set & Samples & Features & Classes & Reference\\
		\midrule
		GyrB           & 52    &   -   &  3 & \cite{tsuda05a}\\
		Digits         & 300   &  16   &  3 & \cite{ucidb07a}\\
		Wine           & 178   &  13   & 13 & \cite{ucidb07a}\\
		Ionosphere     & 351   &  33   &  2 & \cite{ucidb07a}\\
		Balance Scale  & 625   &   4   &  3 & \cite{ucidb07a}\\
		Iris           & 150   &   4   &  3 & \cite{ucidb07a}\\
		Soybean        & 532   &  35   & 17 & \cite{ucidb07a}\\
		USPS           & 2,007 & 256   & 10 & \cite{lecun90c}\\
		Isolet         & 7,797 & 617   & 26 & \cite{ucidb07a}\\
		Prostate       &  322  & 15,154 & 2 & \cite{petricoin02a}\\		
		\bottomrule
	\end{tabular}
	\centering
	\caption{Considered datasets}
	\label{tab:datasets}
\end{table}

In this section, we illustrate the potential of the proposed algorithms on several benchmark experiments. First, the proposed algorithms are evaluated on toy data. Then, they are compared to state-of-the-art kernel learning and Mahalanobis distance learning algorithms on real datasets. Overall, the experiments support that a joint estimation of a subspace and low-dimensional distance in that subspace is a major advantage of the proposed algorithms over methods that estimate the matrix for a subspace that is fixed beforehand. 

Table~\ref{tab:datasets} summarizes the different datasets that have been considered. As a normalization step, the data features are centered and rescaled to unit standard deviation. 

The implementation of the proposed algorithms,\footnote{The source code is available from \texttt{http://www.montefiore.ulg.ac.be/\textasciitilde meyer}} as well as the experiments of this paper are performed with Matlab. The implementations of algorithms MVU,\footnote{\texttt{http://www.cse.wustl.edu/\textasciitilde kilian/Downloads/MVU.html}} KSR,\footnote{\texttt{http://www.cs.uiuc.edu/homes/dengcai2/SR/}} LMNN,\footnote{\texttt{http://www.cse.wustl.edu/\textasciitilde kilian/Downloads/LMNN.html}} and ITML,\footnote{\texttt{http://www.cs.utexas.edu/users/pjain/itml/}} have been rendered publicly available by \cite{weinberger04a}, \cite{cai07a}, \cite{weinberger09a} and \cite{davis07a} respectively. Algorithms POLA \citep{shalev-shwartz04a}, LogDet-KL \citep{kulis09a} and LEGO \citep{jain08a} have been implemented on our own.

\subsection{Toy Data}
In this section, the proposed algorithms are evaluated on synthetic regression problems. The data vectors $\vec{x}_1,...,\vec{x}_n\in\mathbb{R}^d$ and the target matrix $\mat{W}^*\in\FixedRank$ are generated with entries drawn from a standard Gaussian distribution $\mathcal{N}(0,1)$. Observations follow
\begin{equation}\label{eq:obsmodel}
y_i = (\vec{x}_i^T \mat{W}^{*} \vec{x}_i) (1 + \nu_i), \quad i = 1,...,n, 
\end{equation}
where $\nu_i$ is drawn from $\mathcal{N}(0,0.01)$. A multiplicative noise model is preferred over an additive one to easily control that observations remain nonnegative after the superposition of noise. 

\subsubsection{Learning the Subspace vs. Fixing the Subspace Up Front}
\begin{figure}[!hp]
	\centering
	\begin{tabular}{c}
		\includegraphics[width=0.60\textwidth]{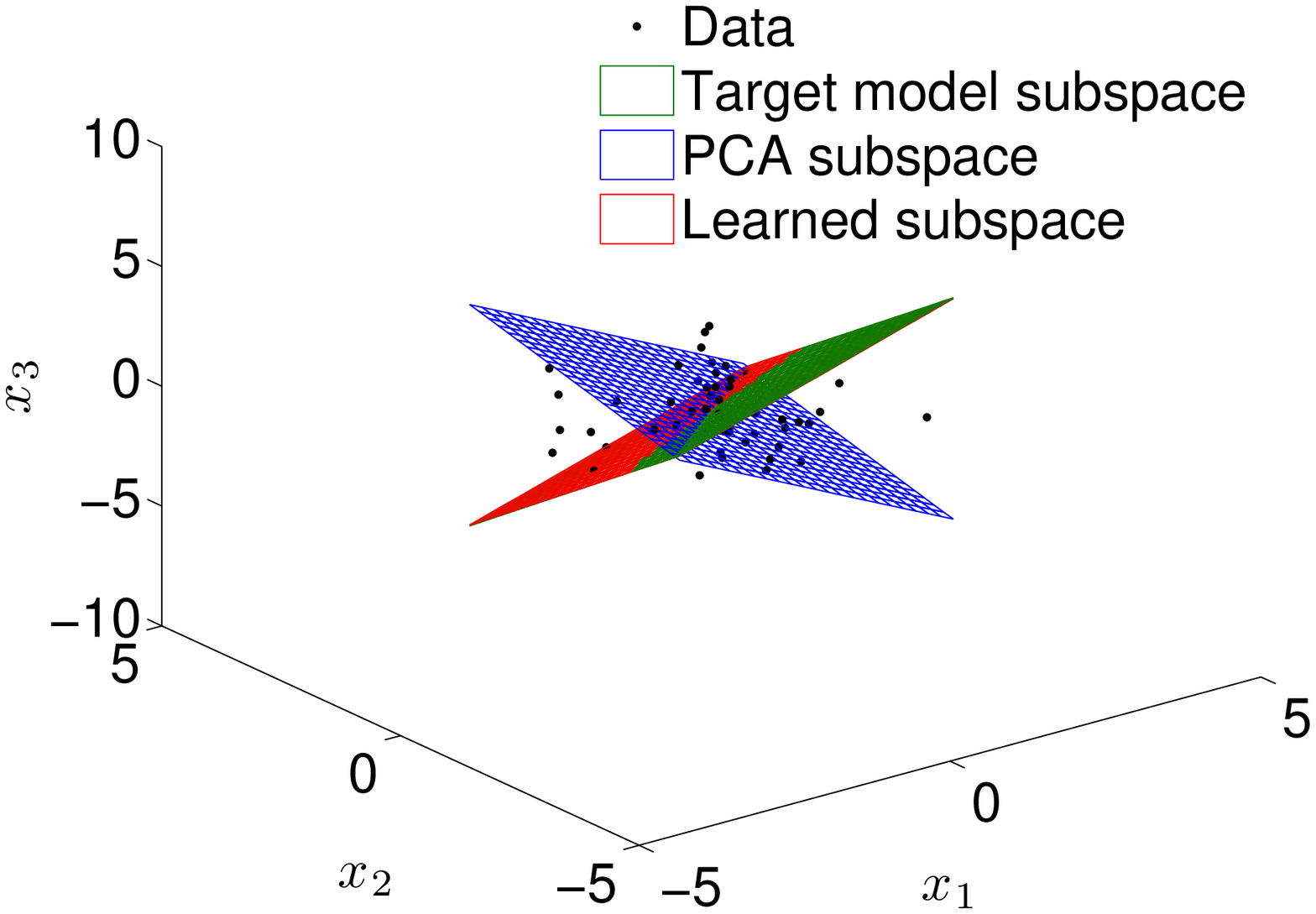}
	\end{tabular}
	\begin{tabular}{cc}
		\includegraphics[width=0.48\textwidth]{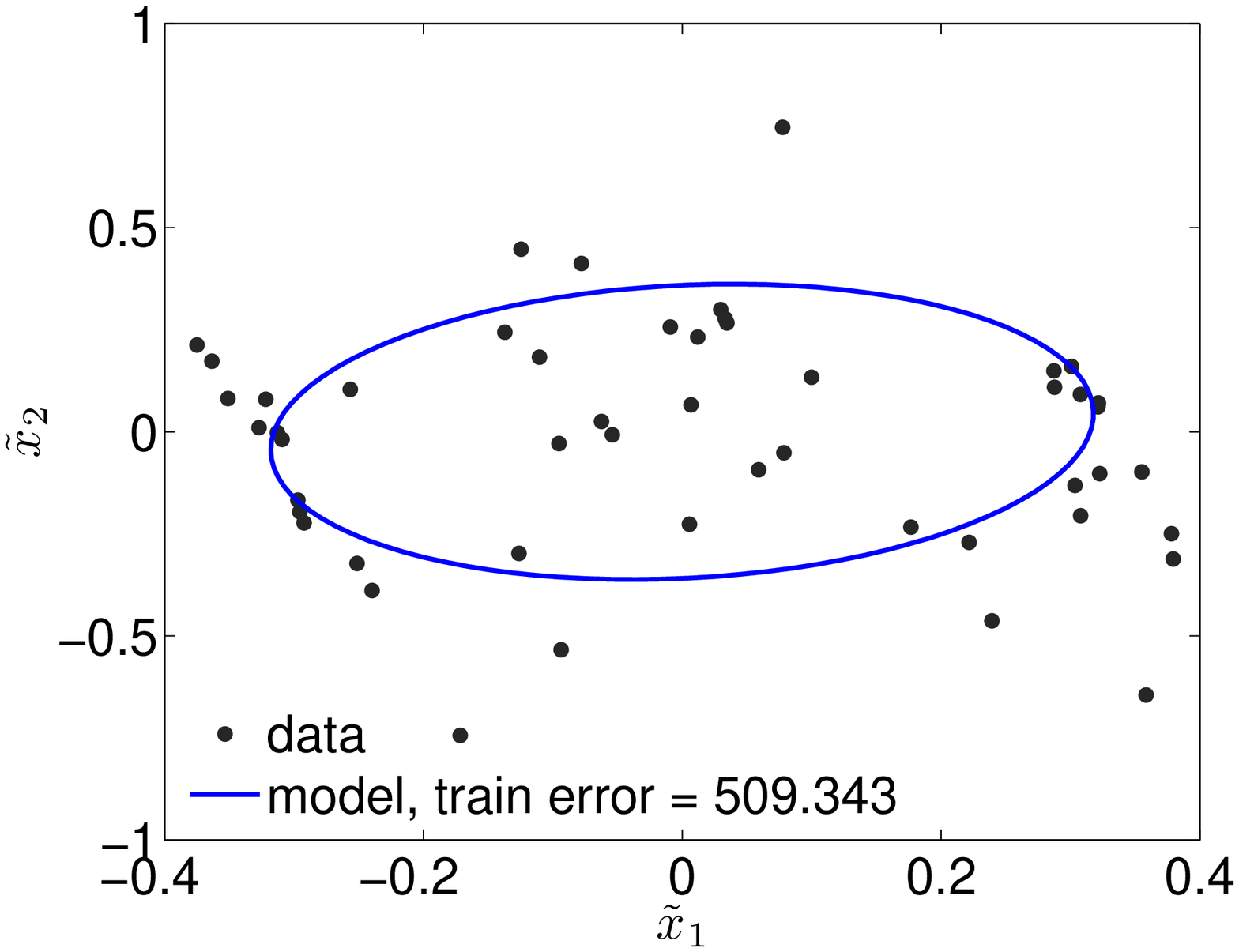} &
	\includegraphics[width=0.48\textwidth]{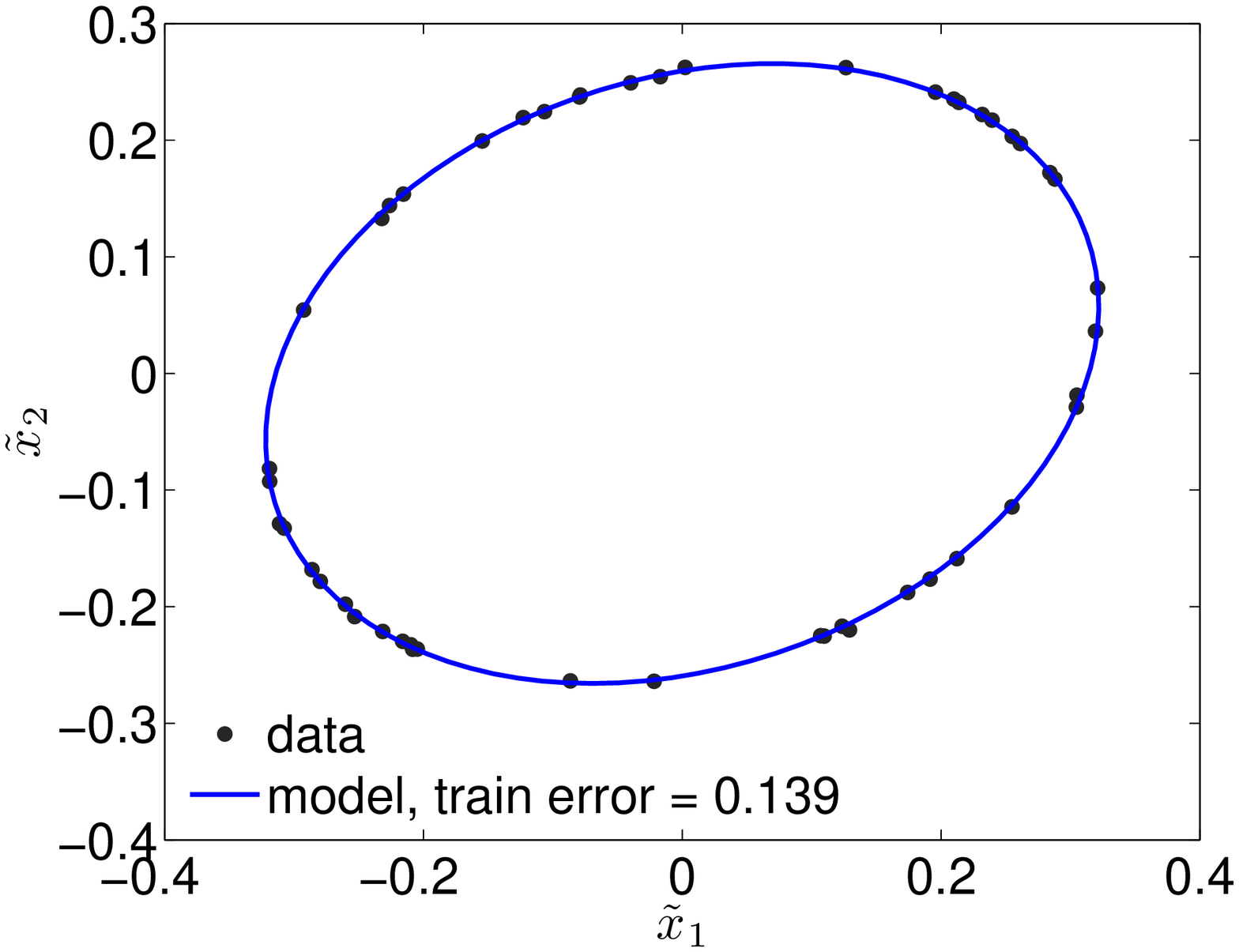}\\
	\end{tabular}
	\caption{Learning vs fixing the subspace. {\bf Top}: the learned subspace is very different from the subspace computed from a classical heuristic. {\bf Bottom left}: fit after projection of the data onto a subspace fixed up front. {\bf Bottom right}: fit obtained with a join estimation of the subspace and a distance within that subspace.}
	\label{fig:illustr-example}
\end{figure}
As an illustrative example, we show the difference between two approaches for fitting the data to observations when a target model $\mat{W}^*\in S_+(3,3)$ is approximated with a parameter $\mat{W}\in S_+(2,3)$. 

A naive approach to tackle that problem is to first project the data $\vec{x}_i\in\mathbb{R}^3$ on a subspace of reduced dimension and then to compute a full-rank model based on the projected data. Recent methods compute that subspace of reduced dimension using principal component analysis~\citep{davis08a,weinberger09a}, that is, a subspace that captures a maximal amount of variance in the data. However, in general, there is no reason why the subspace spanned by the top principal components should coincide with the subspace that is defined by the target model. Therefore, a more appropriate approach consists in learning jointly the subspace and a distance in that subspace that best fits the data to observations within that subspace. 

To compare the two approaches, we generate a set of learning samples $\{(\vec{x}_i,y_i)\}_{i=1}^{200}$, with $\vec{x}_i\in\mathbb{R}^3$ and $y_i$ that follows \eqref{eq:obsmodel}. The target model is
\begin{equation*}
	\mat{W}^* = \tilde{\mat{U}}\boldsymbol\Lambda\tilde{\mat{U}}^T
\end{equation*}
where $\tilde{\mat{U}}$ is a random $3\times 3$ orthogonal matrix and $\boldsymbol\Lambda$ is a diagonal matrix with two dominant values $\Lambda_{11},\Lambda_{22} \gg \Lambda_{33} > 0$ (for this specific example, $\Lambda_{11} = 4, \Lambda_{22} = 3$ and $\Lambda_{33} = 0.01$). Observations $y_i$ are thus nearly generated by a rank-2 model, such that $\mat{W}^{*}$ should be well approximated with a matrix $\mat{W}\in S_+(2,3)$ that minimizes the train error.

Results are presented in Figure~\ref{fig:illustr-example}. The top plot shows that the learned subspace (which identifies with the target subspace) is indeed very different from the subspace spanned by the top two principal components. Moreover, the bottom plots clearly demonstrate that the fit is much better when the subspace and the distance in that subspace are learned jointly. The difference is also significant in terms of the train error. This simple example shows that heuristic methods that fix the range space in the first place may converge to a solution that is very different from a minimum of the desired cost function. For visualization purpose, the two dimensional model is represented by the ellipse
\begin{equation*}
	{\cal E} = \{\tilde{\vec{x}}_i\in\mathbb{R}^2: \tilde{\vec{x}}_i^T\mat{R}^2\tilde{\vec{x}}_i=1\},
	\qquad \mathrm{where}\quad 
	\tilde{\vec{x}}_i=\frac{\mat{U}^T\vec{x}_i}{\sqrt{y_i}},
\end{equation*}
and $(\mat{U},\mat{R}^2)$ are computed with algorithm \eqref{eq:batch-polar}, either in the setting $\lambda=0$ that fixes the subspace to the PCA subspace (left) or in the setting $\lambda=0.5$ that simultaneously learned $\mat{U}$ and $\mat{B}$ (right). A perfect fit is obtained when all $\tilde{\vec{x}}_i$ are located on ${\cal E}$, which is the locus of points where $\hat{y}_i=y_i$.

\subsubsection{Influence of $\lambda$ on the Algorithm Based on the Polar Geometry} 
\begin{figure}[!ht]
	\centering
	\begin{tabular}{cc}
		\includegraphics[width=0.48\textwidth]{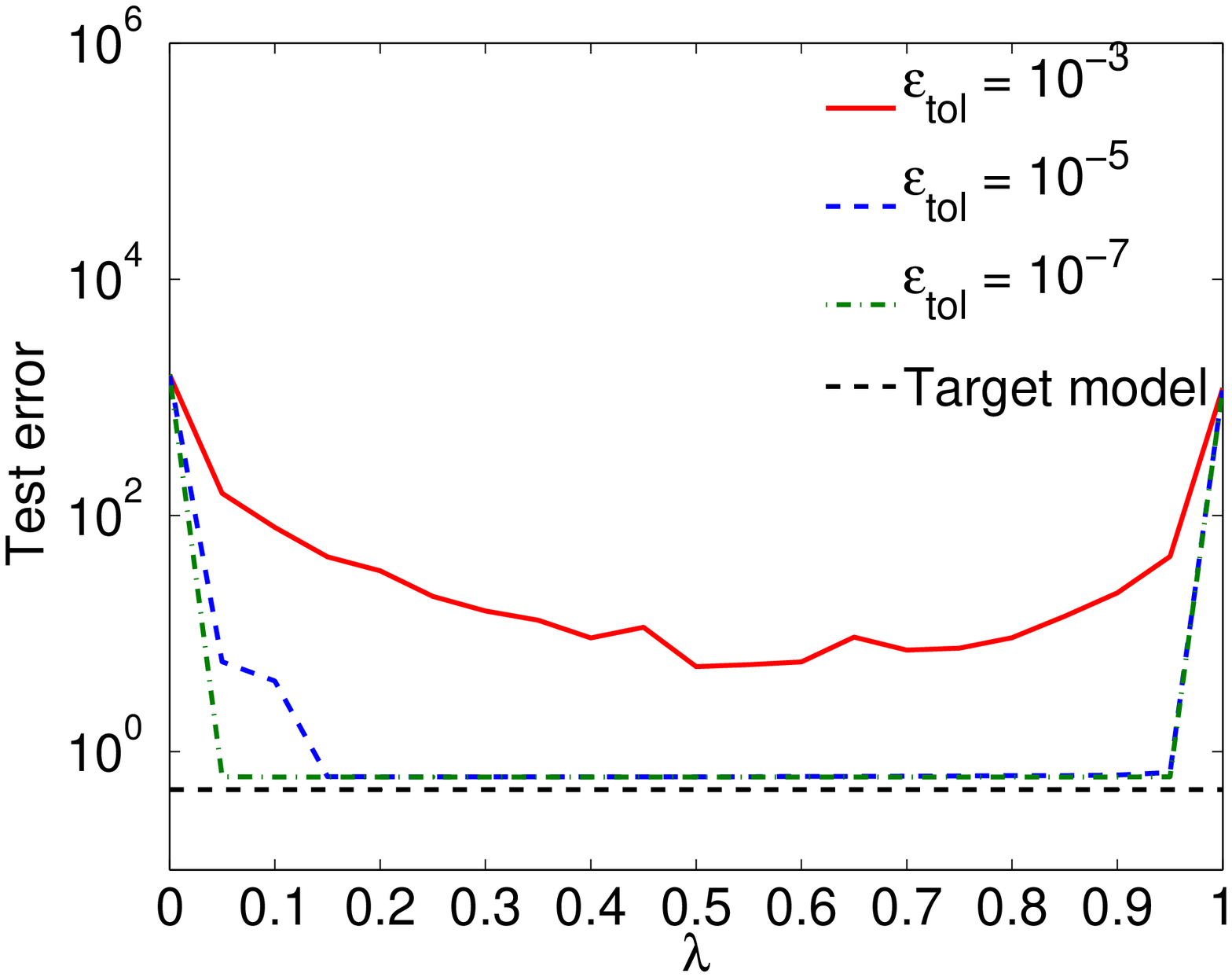} &
		\includegraphics[width=0.48\textwidth]{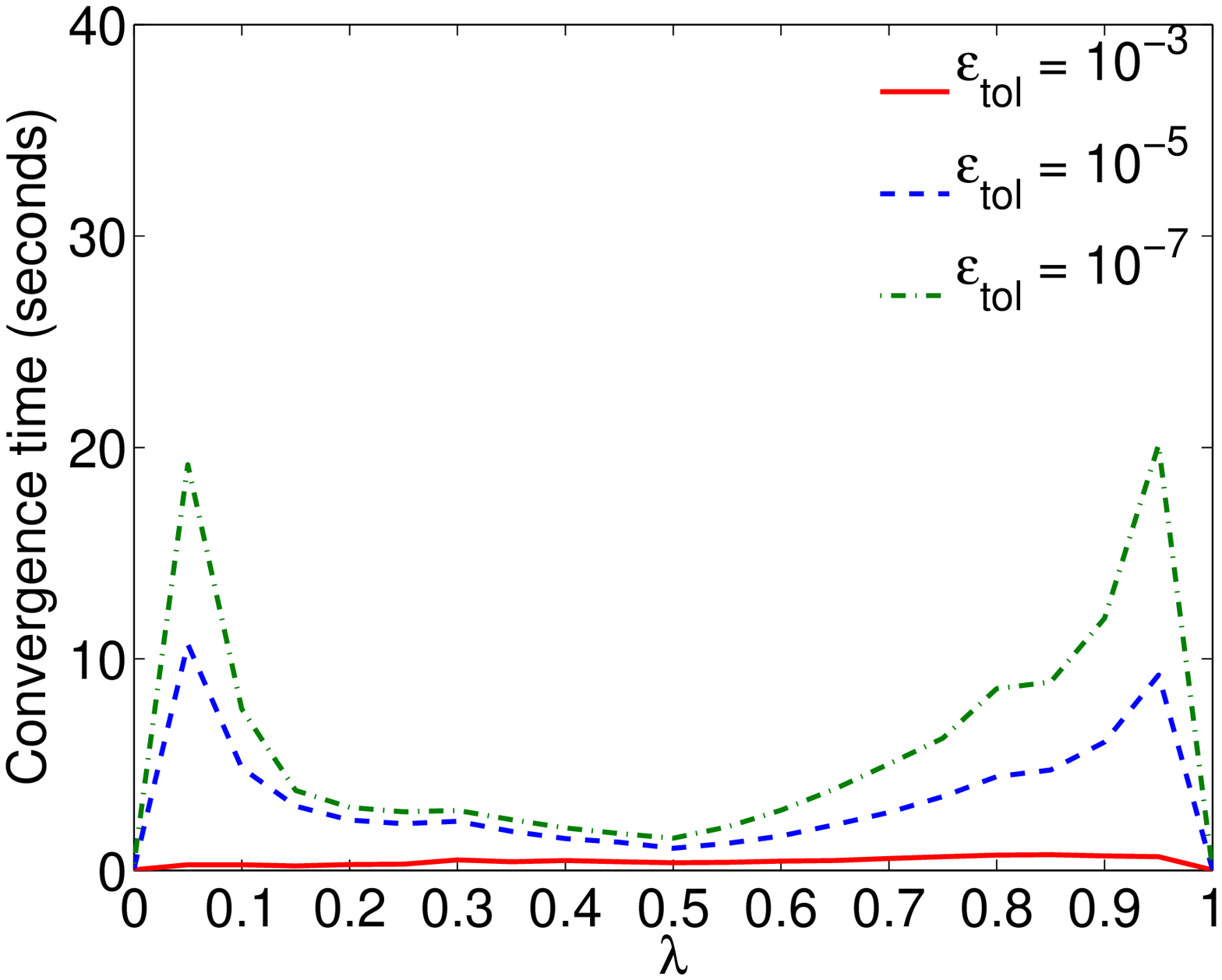}\\
	\end{tabular}
	\caption{Influence of $\lambda$.}
	\label{fig:lambda}
\end{figure}
In theory, the parameter $\lambda$ should not influence the algorithm since it has no effect on the first-order optimality conditions except for its two extreme values $\lambda = 0$ and $\lambda = 1$. In practice however, a sensitivity to this parameter is observed due to the finite tolerance of the stopping criterion: the looser the tolerance, the more sensitive to $\lambda$. 

To investigate the sensitivity to $\lambda$, we try to recover a target parameter $\mat{W}^*\in S_+(5,10)$ using pairs $(\vec{x}_i,y_i)$ generated according to \eqref{eq:obsmodel}. We generate $10$ random regression problems with $1000$ samples partitioned into $500$ learning samples and $500$ test samples. We compute the mean test error and the mean convergence time as a function of $\lambda$ for different values of $\epsilon_{tol}$. The results are presented in Figure \ref{fig:lambda}. As $\epsilon_{tol}$ decrease, the test error becomes insensitive to $\lambda$, but an influence is observed on the convergence time of the algorithm.

In view of these results, we recommend the value $0.5$ as the default setting for $\lambda$. Unless specified otherwise, we therefore use this particular value for all experiments in this paper.

\subsubsection{Online vs. Batch} 
\begin{figure}[!ht]
	\centering
	\begin{tabular}{cc}
		\includegraphics[width=0.48\textwidth]{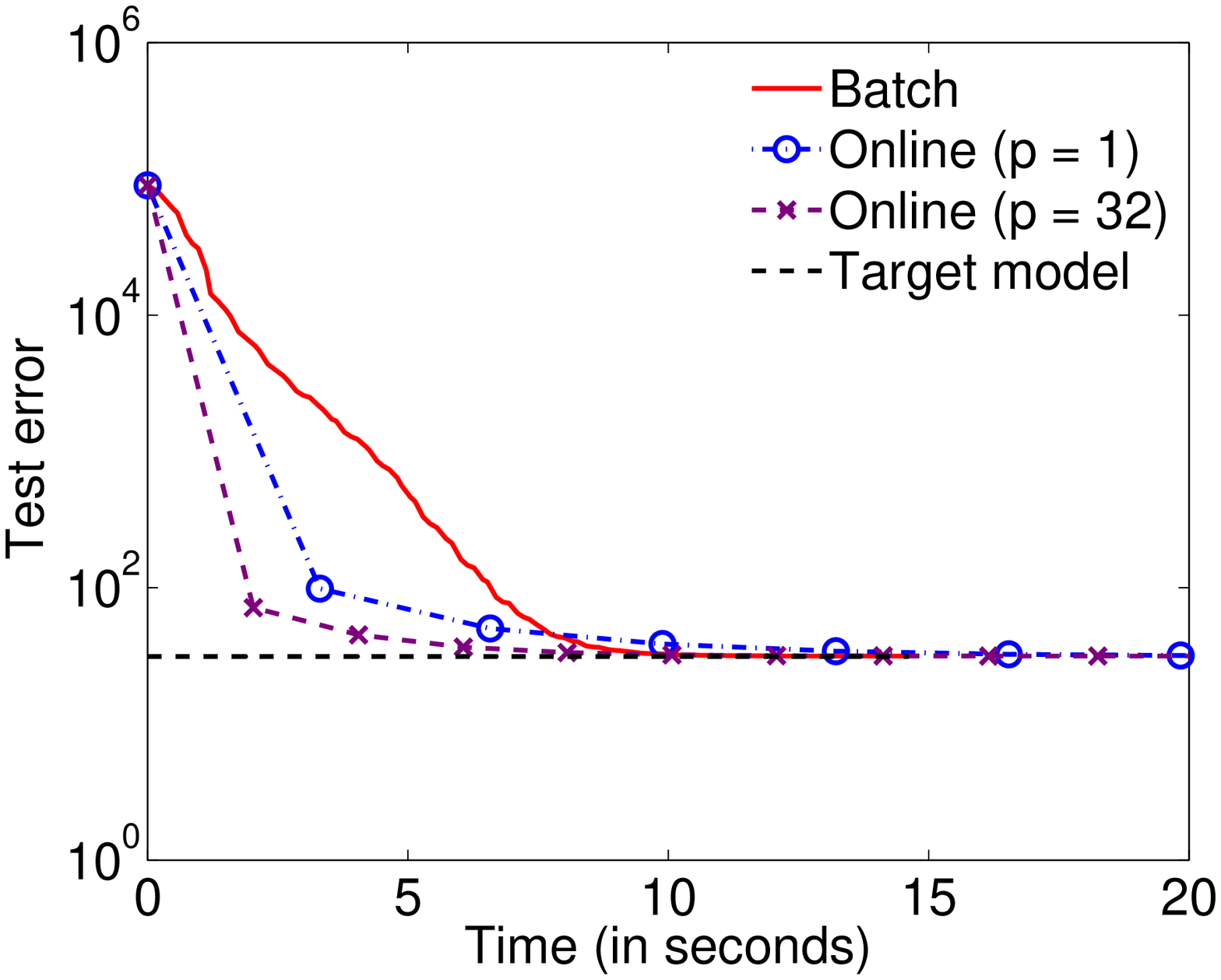} &
		\includegraphics[width=0.48\textwidth]{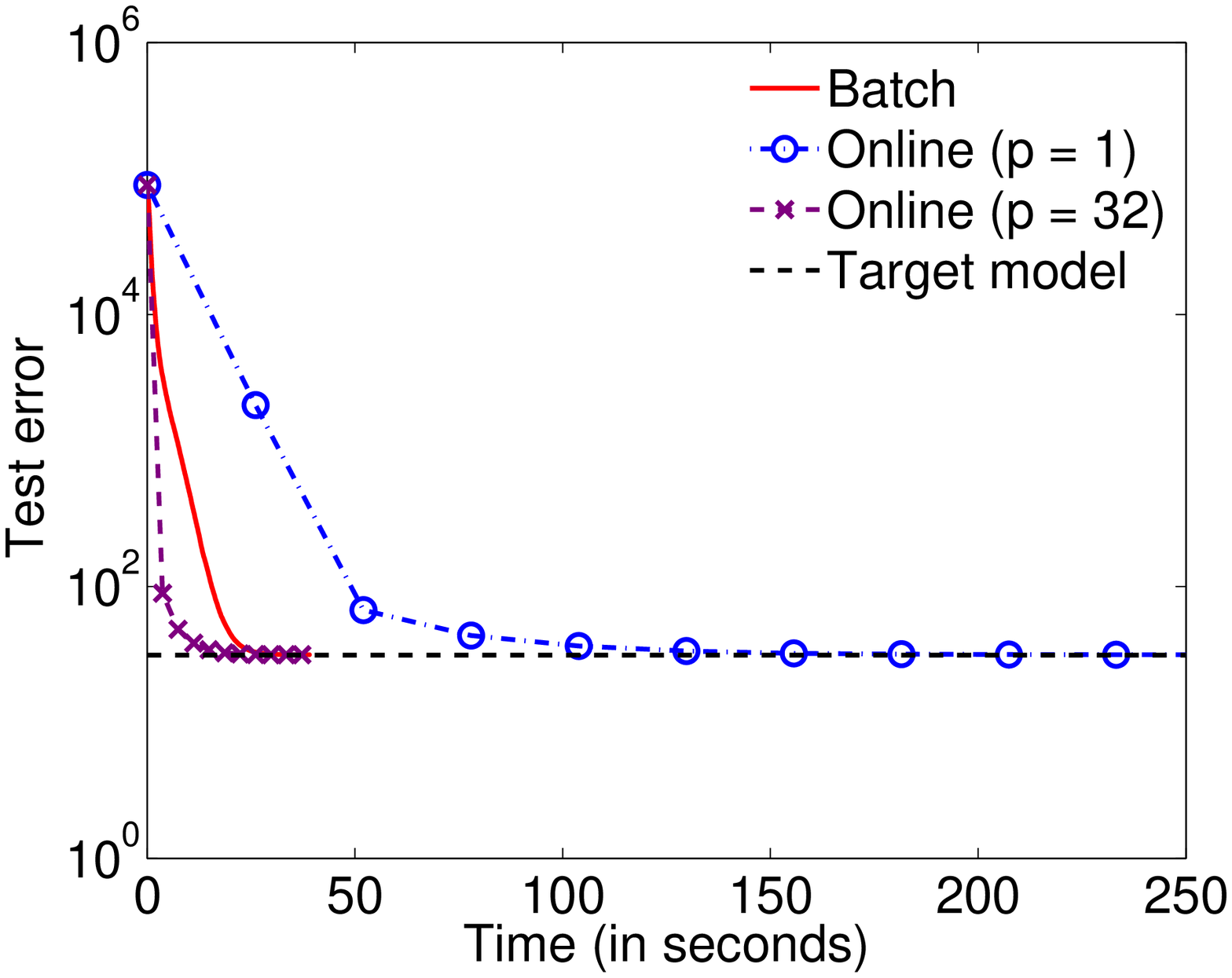}\\
		(a) Flat geometry & (b) Polar geometry\vspace{5mm}\\
	\end{tabular}
	\caption{Online vs Batch. For a large number of samples, online algorithms reduce the test error much more rapidly than batch ones. Using the mini-batch extension generally improve significantly the performance.}
	\label{fig:online-batch}
\end{figure}
This experiment shows that when a large amount of sample is available ($80,000$ training samples and $20,000$ test samples for learning a parameter $\mat{W}^*$ in $S_+(10,50)$), online algorithms minimize the test error more rapidly than batch ones. It further shows that the mini-batch extension allows to improve significantly the performance compared to the plain stochastic gradient descent setting ($p=1$). We observe that the mini-batch size $p = 32$ generally gives good results. Figure~\ref{fig:online-batch} report the test error as a function of the learning time, that is, the time after each iteration for batch algorithm and the time after each epoch for online algorithms. For the algorithm based on the polar geometry, the mini-batch extension is strongly recommended to amortize the larger cost of each update. 

\subsection{Kernel Learning}\label{sec:exp-kernel-learning}
In this section, the proposed algorithms are applied to the problem of learning a kernel matrix from pairwise distance constraints between data samples. As mentioned earlier, we only consider this problem in the transductive setting, that is, all samples $\vec{x}_1,...\vec{x}_n$ are available up front and the learned kernel do not generalize to new samples. 

\subsubsection{Experimental Setup}
After transformation of the data with the kernel map $\vec{x}\mapsto\phi(\vec{x})$, the purpose is to compute a fixed-rank kernel matrix based on a limited amount of pairwise distances in the kernel feature space and on some information about class labels.

Distance constraints are generated as $\hat{y}_{ij} \leq y_{ij} (1-\alpha)$ for identically labeled samples and $\hat{y}_{ij} \geq y_{ij} (1+\alpha)$ for differentially labeled samples, where $\alpha\geq 0$ is a scaling factor, $y_{ij}=\|\phi(\vec{x}_i)-\phi(\vec{x}_j)\|^2$ and $\hat{y}_{ij}=\trace(\mat{W}(\vec{e}_i-\vec{e}_j)(\vec{e}_i-\vec{e}_j)^T)=(\vec{e}_i-\vec{e}_j)^T \mat{W} (\vec{e}_i-\vec{e}_j)$.

We investigate both the influence of the amount of side-information provided, the influence of the approximation rank and the computational time required by the algorithms.

To quantify the performance of the learned kernel matrix, we perform either a classification or a clustering of the samples based on the learned kernel. For classification, we compute the test set accuracy of a $k$-nearest neighbor classifier ($k=5$) using a two-fold cross-validation protocol (results are averaged over 10 random splits). For clustering, we use the $K$-means algorithm with the number of clusters equal to the number of classes in the problem. To overcome K-means local minima, $10$ runs  are performed in order to select the result that has lead to the smaller value of the K-means objective. The quality of the clustering is measured by the normalized mutual information (NMI) shared between the random variables of cluster indicators $C$ and target labels $T$ \citep{strehl00a},
\begin{equation*}
	NMI = \frac{2\ I(C;T)}{(H(C) + H(T))},
\end{equation*}
where $I(X_1;X_2)=H(X_1) - H(X_1|X_2)$ is the mutual information between the random variables $X_1$ and $X_2$, $H(X_1)$ is the Shannon entropy of $X_1$, and $H(X_1|X_2)$ is the conditional entropy of $X_1$ given $X_2$. This score ranges from 0 to 1, the larger the score, the better the clustering quality.

\subsubsection{Compared Methods}
We compare the following methods:
\begin{enumerate}
	\item Batch algorithms \eqref{eq:batch-flat} and \eqref{eq:batch-polar}, adapted to handle inequalities (see Section \ref{sec:inequalities}),
	\item The kernel learning algorithm LogDet-KL \citep{kulis09a} which learn kernel matrices of fixed range space for a given set of distance constraints.
	\item The kernel spectral regression (KSR) algorithm of \cite{cai07a} using a similarity matrix $\mat{N}$ constructed as follows. Let $\mat{N}$ be the adjacency matrix of a $5$-NN graph based on the initial kernel. We modify $\mat{N}$ according to the set of available constraints: $\mat{N}_{ij}=1$ if samples $\vec{x}_i$ and $\vec{x}_j$ belong to the same class (must-link constraint), $\mat{N}_{ij}=0$ if samples $\vec{x}_i$ and $\vec{x}_j$ do not belong to the same class (cannot-link constraint).
	\item The Maximum Variance Unfolding (MVU) algorithm \citep{weinberger04a},
	\item The Kernel PCA algorithm \citep{scholkopf98a}.
\end{enumerate}
The last two algorithms are unsupervised techniques that are provided as baselines.

\begin{figure}[!th]
\centering
\begin{tabular}{cc}
	\includegraphics[width=0.48\textwidth]{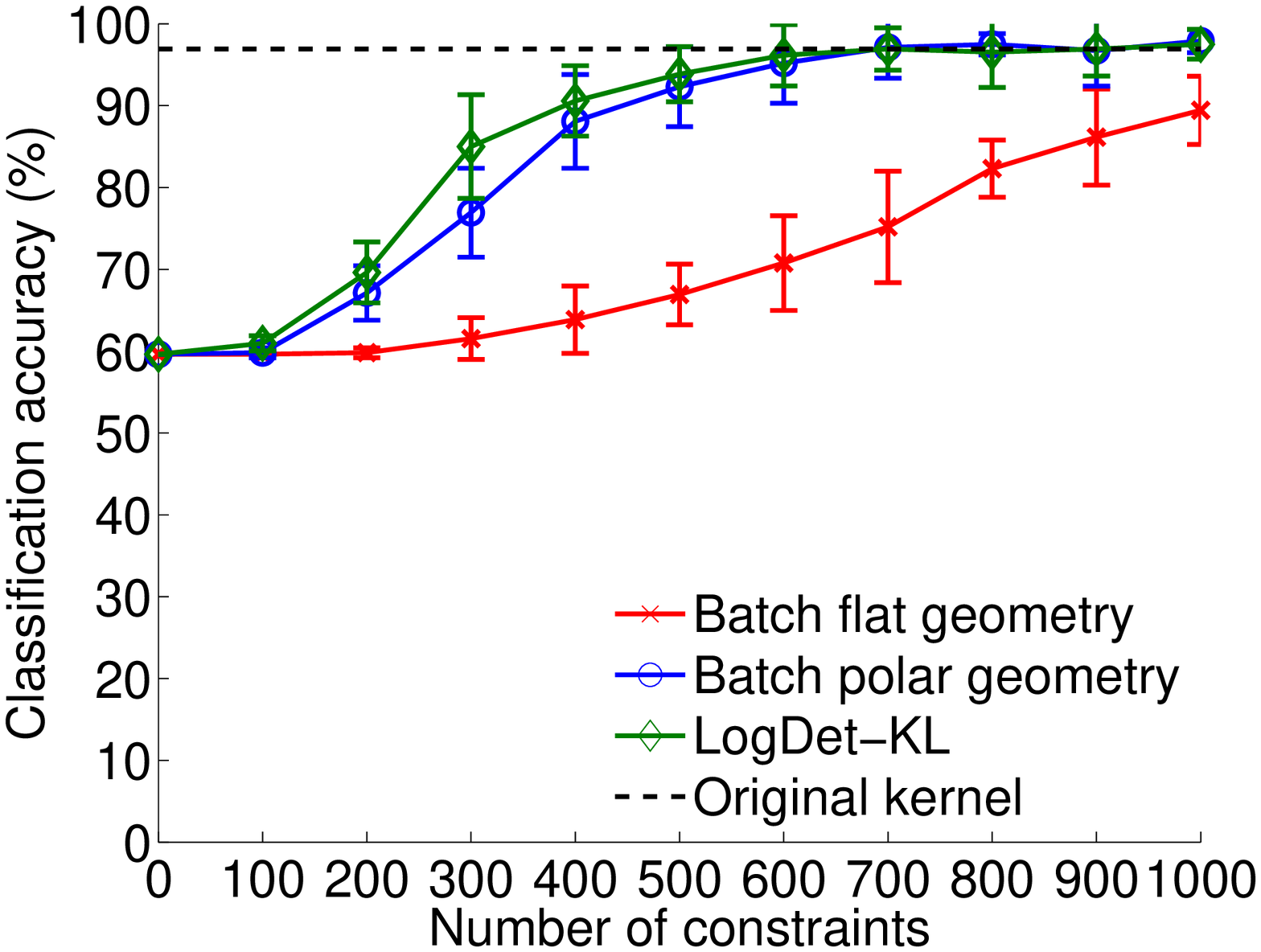} &
	\includegraphics[width=0.48\textwidth]{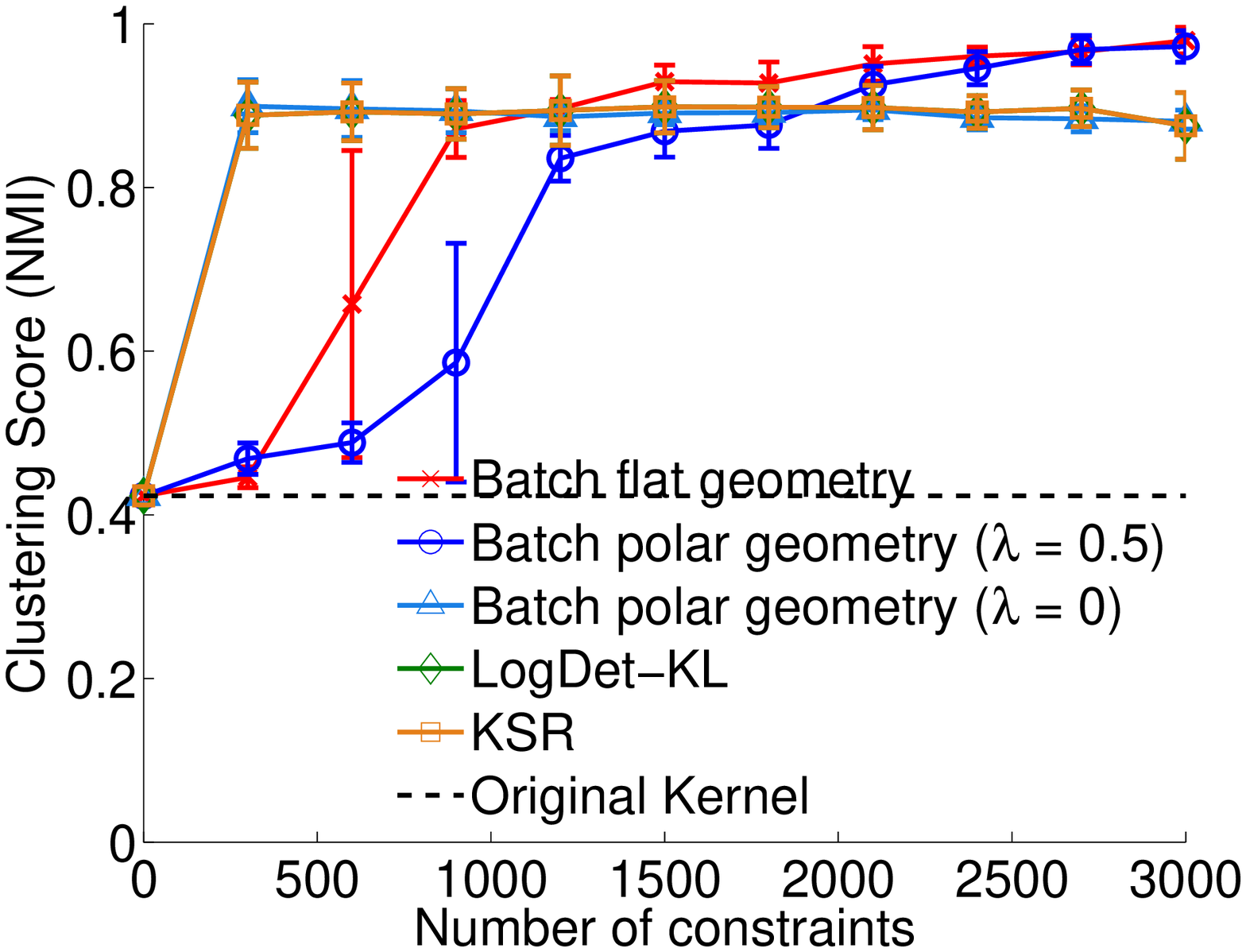}\\
\end{tabular}
\caption{{\bf Left}: full-rank kernel learning on the Gyrb data set. The algorithm based on the polar geometry competes with LogDet-KL. {\bf Right}: low-rank kernel learning on the Digits data set. The proposed algorithms outperform the compared methods as soon as a sufficiently large number of constraints is provided.}
\label{fig:gyrb-digits}
\end{figure}
\subsubsection{Results}

The first experiment is reproduced from \cite{tsuda05a} and \cite{kulis09a}. The goal is to reconstruct the GyrB kernel matrix based on distance constraints only. This matrix contains information about the proteins of three bacteria species. The distance constraints are randomly generated from the original kernel matrix with $\alpha=0$. We compare the proposed batch methods with the LogDet-KL algorithm, the only competing algorithm that also learns directly from distance constraints. This algorithm is the best performer reported by \cite{kulis09a} for this experiment. All algorithms start from the identity matrix that do not encode any domain information. Figure~\ref{fig:gyrb-digits} (left) reports the $k$-NN classification accuracy as a function of the number of distance constraints provided. In this full-rank learning setting, the algorithm based on the polar geometry compete with the LogDet-KL algorithm. The convergence time of the algorithm based on the polar geometry is however much faster (0.15 seconds versus 58 seconds for LogDet-KL when learning $1000$ constraints). The algorithm based on the flat geometry has inferior performance when too few constraints are provided. This is because in the kernel learning setting, updates of this algorithm only involve the rows and columns that correspond to the set of points for which constraints are provided. It may thus result in a partial update of the kernel matrix entries. This issue disappears as the number of provided constraints increases.

The second experiment is reproduced from \cite{kulis09a}. It aims at improving an existing low-rank kernel using limited information about class labels. A rank-16 kernel matrix is computed for clustering a database of $300$ handwritten digits randomly sampled from the 3, 8 and 9 digits of the Digits dataset (since we could not find out the specific samples that have been selected by \cite{kulis09a}, we made our own samples selection). The distance constraints are randomly sampled from a linear kernel on the input data $\mat{K}=\mat{X}\mat{X}^T$ and $\alpha=0.25$. The results are presented in Figure~\ref{fig:gyrb-digits} (right). The figure shows that KSR, LogDet-KL and the algorithm based on the polar geometry with $\lambda=0$ perform similarly. These methods are however outperformed by the proposed algorithms (flat geometry and polar geometry with $\lambda=0.5$) when the number of constraints is large enough. This experiment also enlightens the flexibility of the polar geometry, which allows us to fix the subspace in situations where too few constraints are available.

\begin{figure}[!th]
\centering
\begin{tabular}{cc}
 \includegraphics[width=0.48\textwidth]{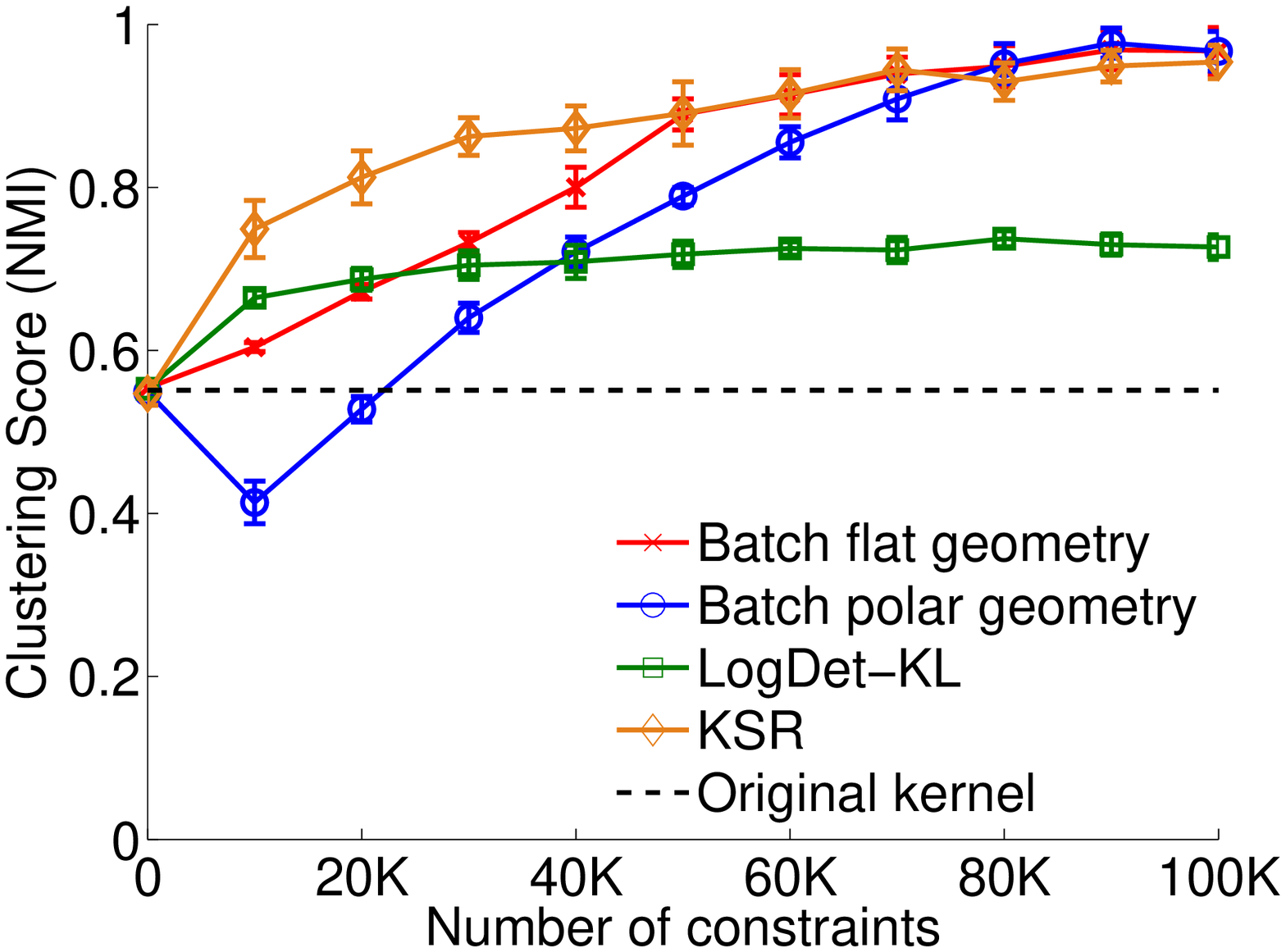} &
 \includegraphics[width=0.48\textwidth]{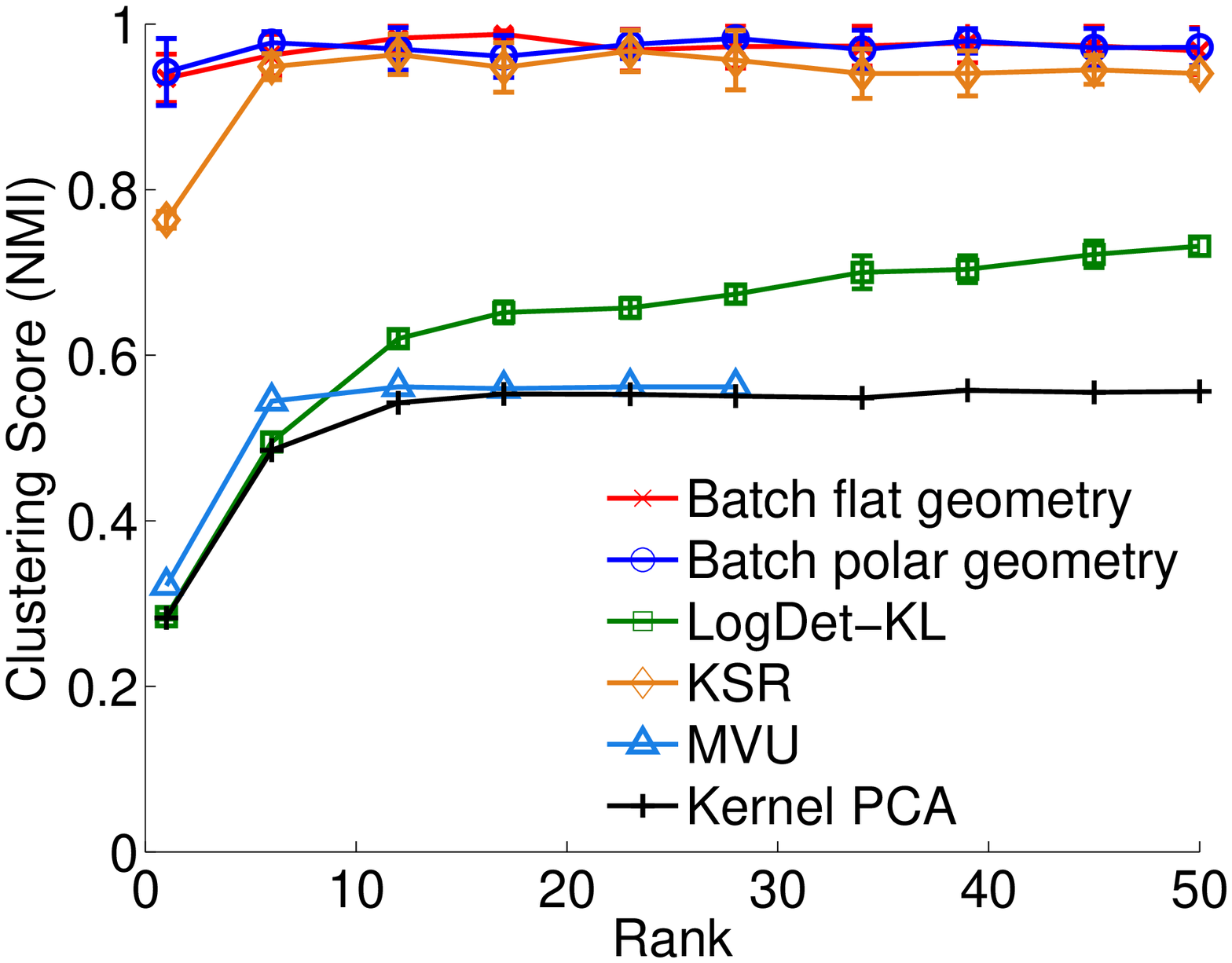}\\
\end{tabular}
\caption{Clustering the USPS data set. {\bf Left:} clustering score versus number of constraints. {\bf Right:} clustering score versus approximation rank. When the number of provided constraints is large enough, the proposed algorithms perform as good as the KSR algorithm. It outperforms the LogDet-KL algorithm and baselines.}
\label{fig:usps}
\end{figure}

Finally, we tackle the kernel learning problem on a larger data set. We use the test set of the USPS dataset,\footnote{We use the ZIP code data from \texttt{http://www-stat-class.stanford.edu/\textasciitilde tibs/ElemStatLearn/data.html} .} which contains $2007$ samples of handwritten zip code digits. The data are first transformed using the kernel map $\kappa(\vec{x}_i,\vec{x}_j)=\exp(-\gamma\|\vec{x}_i-\vec{x}_j\|_2^2)$ with $\gamma=0.001$ and we further center the data in the kernel feature space. Pairwise distance constraints are randomly sampled from that kernel matrix with $\alpha=0.5$. Except KSR that has its own initialization procedure, algorithms start from the kernel matrix provided by kernel PCA.

Figure \ref{fig:usps} (left) shows the clustering performance as a function of the number of constraints provided when the approximation rank is fixed to $r=25$. Figure \ref{fig:usps} (right) reports the clustering performance as a function of the approximation rank when the number of constraints provided is fixed to $100K$. When the number of provided constraints is large enough, the proposed algorithms perform as good as KSR and outperform the LogDet-KL method that learn a kernel of fixed-range space. Average computational times for learning a rank-$6$ kernel from $100K$ constraints are $0.57$ seconds for KSR, $3.25$ seconds for the algorithm based on the flat geometry, $46.78$ seconds for LogDet-KL and $47.30$ seconds for the algorithm based on the polar geometry. In comparison, the SDP-based MVU algorithm takes $676.60$ seconds to converge.

\subsection{Mahalanobis Distance Learning}\label{sec:exp-distance-learning}
In this section, we tackle the problem of learning from data a Mahalanobis distance for supervised classification and compare our methods to state-of-the-art Mahalanobis metric learning algorithms.

\subsubsection{Experimental Setup}
For the considered problem, the purpose is to learn the parameter $\mat{W}$ of a Mahalanobis distance $d_{\mat{W}}(\vec{x}_i,\vec{x}_j)=(\vec{x}_i-\vec{x}_j)^T \mat{W} (\vec{x}_i-\vec{x}_j)$, such that the distance satisfies as much as possible a given set of constraints. As in the paper of \cite{davis07a}, we generate the constraints from the learning set of samples as $d_{\mat{W}}(\vec{x}_i,\vec{x}_j)\leq l$ for same-class pairs and $d_{\mat{W}}(\vec{x}_i,\vec{x}_j)\geq u$ for different-class pairs. The scalars $u$ and $l$ estimate the $95$-th and $5$-th percentiles of the distribution of Mahalanobis distances parameterized by a chosen baseline $\mat{W}_0$. The performance of the learned distance is then quantified by the test error rate of a $k$-nearest neighbor classifier based on the learned distance. All experiments use the setting $k=5$, breaking ties arbitrarily. Unless for the Isolet data set for which a specific train/test partition is provided, error rates are computed using two-fold cross validation. Results are averaged over $10$ random partitions.

\subsubsection{Compared Methods}
We compare the following distance learning algorithms: 
\begin{enumerate}
	\item Batch algorithms \eqref{eq:batch-flat} and \eqref{eq:batch-polar},
	\item ITML \citep{davis07a},
	\item LMNN \citep{weinberger09a},
	\item Online algorithms \eqref{eq:online-flat} and \eqref{eq:online-polar},
	\item LEGO \citep{jain08a},
	\item POLA \citep{shalev-shwartz04a}.
\end{enumerate}
When some methods require the tuning of an hyper-parameter, this is performed by a two-fold cross-validation procedure. The slack parameter of ITML as well as the step size of POLA are selected in the range of values $10^k$ with $k=-3,...,3$. The step size of LEGO is selected in this same range of value for the UCI datasets, and in the range of value $10^k$ with $k=-10,...,-5$ for the larger data sets Isolet and Prostate.

\subsubsection{Results}
\begin{figure}[!ht]
\centering
	\includegraphics[angle=-90,width=285pt]{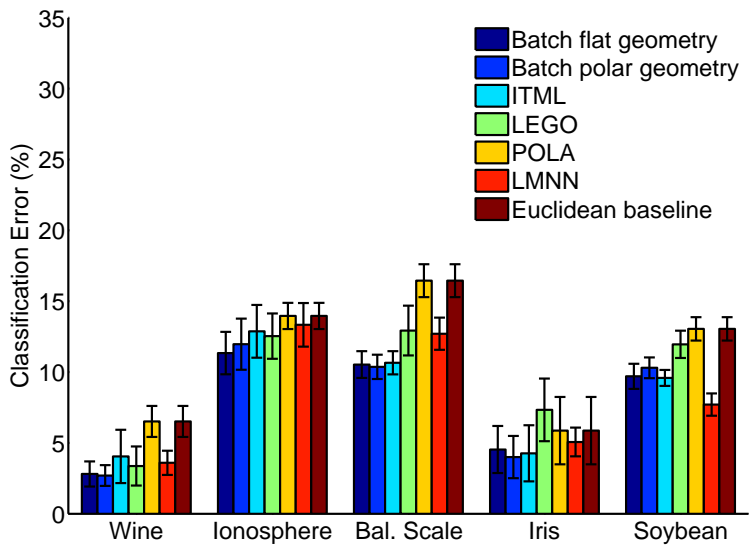}
\caption{Full-rank distance learning on the UCI data sets. The proposed algorithms compete with state-of-the-art methods for learning a full-rank Mahalanobis distance.}
\label{fig:metric-uci}
\end{figure}
Reproducing a classical benchmark experiment from \cite{kulis09a}, we demonstrate that the proposed batch algorithms compete with state-of-the-art full-rank Mahalanobis distance learning algorithms on several UCI datasets (Figure~\ref{fig:metric-uci}). We have not included the online versions of our algorithms in this comparison because we consider that the batch approaches are more relevant on such small datasets. Except POLA and LMNN which do not learn from provided pairwise constraints, all algorithms process $40c(c-1)$ constraints, where $c$ is the number of classes in the data. We choose the Euclidean distance ($\mat{W}_0=\mat{I}$) as the baseline distance for initializing the algorithms. Figure \ref{fig:metric-uci} reports the results. The two proposed algorithms compete favorably with the other full-rank distance learning techniques, achieving the minimal average error for $4$ of the $5$ considered data sets.

\begin{figure}[!ht]
\centering
\begin{tabular}{cc}
	\includegraphics[width=0.48\textwidth]{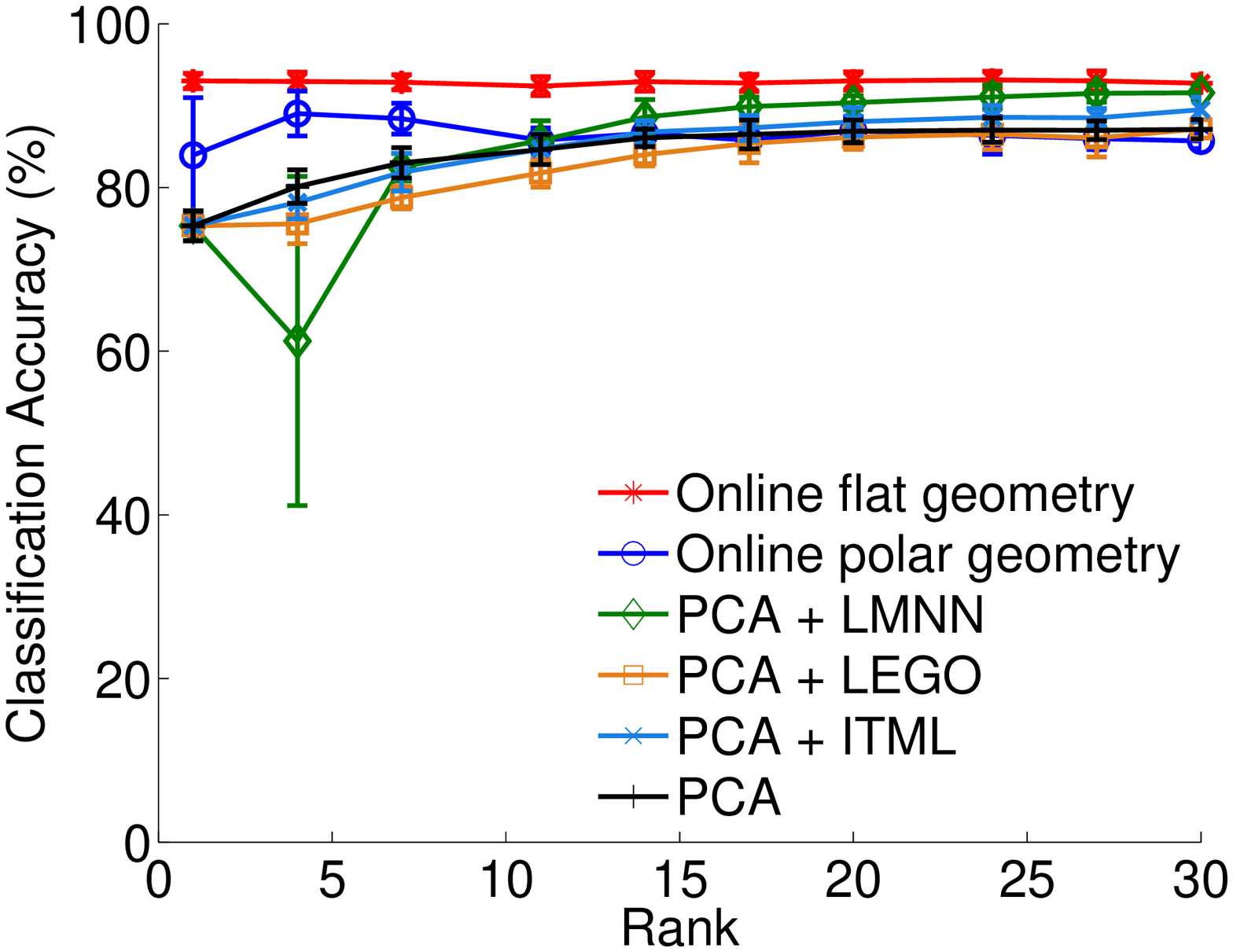} &
	\includegraphics[width=0.48\textwidth]{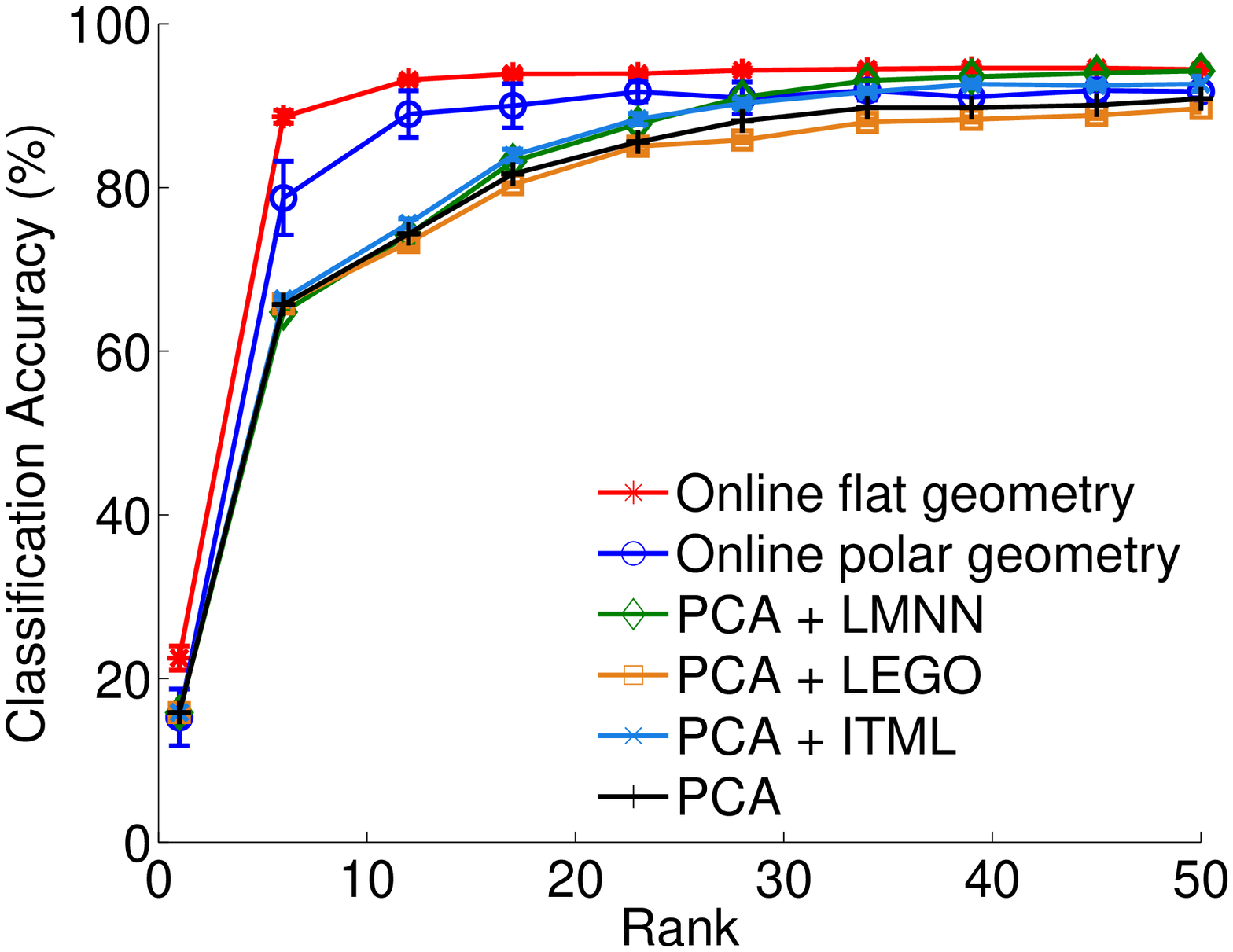} \\
	(a) Prostate & (b) Isolet\\
\end{tabular}
\caption{Low-rank Mahalanobis distance learning. For low values of the rank, the proposed algorithms perform much better than the methods that project the data on the top principal directions and learn a full-rank distance on the projected data.}
\label{fig:large-scale}
\end{figure}

We finally evaluate the proposed algorithms on higher-dimensional data sets in the low-rank regime (Figure \ref{fig:large-scale}). The distance constraints are generated as in the full-rank case, but the initial baseline matrix is now computed as $\mat{W}_0=\mat{G}_0\mat{G}_0^T$, where $\mat{G}_0$'s columns are the top principal directions of the data. For the Isolet data set, $100K$ constraints are generated, and $10K$ constraints are generated for the Prostate data set. For scalability reasons, algorithms LEGO, LMNN and ITML must proceed in two steps: the data are first projected onto the top principal directions and then a full-rank distance is learned within the subspace spanned by these top principal directions. In contrast, our algorithms are initialized with the top principal direction, but they operate on the data in their original feature space. Overall, the proposed algorithms achieve much better performance than the methods that first reduce the data. This is particularly striking when the rank is very small compared to problem size. The performance gap reduces as the rank increases. However, for high-dimensional problems, one is usually interested in efficient low-rank approximations that gives satisfactory results. 

\section{Conclusion}
In this paper, we propose gradient descent algorithms to learn a regression model parameterized by a fixed-rank positive semidefinite matrix. The rich Riemannian geometry of the set of fixed-rank PSD matrices is exploited through a geometric optimization approach.

The resulting algorithms overcome the main difficulties encountered by the previously proposed methods as they scale to high-dimensional problems, and they naturally enforce the rank constraint as well as the positive definite property while leaving the range space of the matrix free to evolve during optimization.

We apply the proposed algorithms to the problem of learning a distance function from data, when the distance is parameterized by a fixed-rank positive semidefinite matrix. The good performance of the proposed algorithms is illustrated over several benchmarks.

\section*{Acknowledgements}
This paper presents research results of the Belgian Network DYSCO (Dynamical Systems, Control, and Optimization), funded by the Interuniversity Attraction Poles Programme, initiated by the Belgian State, Science Policy Office. The scientific responsibility rests with its authors. Gilles Meyer is supported as an FRS-FNRS research fellow (Belgian Fund for Scientific Research).

\appendix
\section{Convergence Proof of Algorithm \eqref{eq:online-flat}}\label{sec:proof-sgd}
\cite{bottou98a} reviews the mathematical tools required to prove almost sure convergence, that is asymptotic convergence with probability one, of stochastic gradient algorithms. Almost sure convergence follows from the following five assumptions:
\begin{itemize}
\item[(A1)] $F(\mat{G}) = \mathbb{E}_{\mat{X},y}\{\ell(\hat{y},y)\}\geq 0$ is three times differentiable with bounded derivatives,
\item[(A2)] the step sizes satisfy $\sum_{t=1}^\infty \eta_t^2 < \infty$ and $\sum_{t=1}^\infty \eta_t = \infty$,
\item[(A3)] $\mathbb{E}_{\mat{X},y}\{\|\grad{}{f(\mat{G})}\|^2_F\} \leq  k_1 + k_2 \|\mat{G}\|^2_F$, where $f(\mat{G})=\ell(\hat{y},y)$,
\item[(A4)] $\displaystyle\exists h_1 > 0, \inf_{\|{\mat{G}}\|^2_F > h_1} \trace(\mat{G}^T \mathbb{E}_{\mat{X},y}\{\grad{}{f(\mat{G})}\}) > 0$,
\item[(A5)]   $\displaystyle\exists h_2 > h_1, \forall (\mat{X},y)\in \mathcal{X}\times\mathcal{Y},
                 \sup_{\|\mat{G}\|^2_F < h_2} \|\grad{}{f(\mat{G})}\|_F \leq k_3$,
\end{itemize}
where $\|\cdot\|_F$ is the Frobenius norm. Provided that algorithm~\eqref{eq:online-flat} is equipped with an adaptive step size $s_{t}= \eta_{t}/\max(\|\mat{G}_{t}\|_{F}^{2},1)$, where $\eta_t$ satisfy (A2), we have the following convergence result.

\begin{proposition}
For bounded data $(\mat{X},y)$, algorithm \eqref{eq:online-flat} equipped with the step size $s_{t}$ defined above converges almost surely to the set of stationary points of the cost function $\mathbb{E}_{\mat{X},y}\{(\hat{y}-y)^{2}/2\}$.
\end{proposition}
\begin{proof}
The proof is completed in two steps. First, it is shown that the stochastic sequence
\begin{equation*}
  u_t = \max(h_2,\|{\mat{G}_t}\|^2_F),
\end{equation*}
defines a Lyapunov process (always positive and decreasing on average) which is bounded almost surely by $h_2$. This implies that $\mat{G}_t$ is almost surely confined within distance $\sqrt{h_2}$ from the origin and provides almost sure bounds on all continuous functions of $\mat{G}_t$. In \cite{bottou98a}, confinement is essentially based on (A3) and (A4). In the current proof, we rely on the fact that $\mathbb{E}_{\mat{X},y}\{\|\grad{}{f(\mat{G})}/\max(\|\mat{G}\|_{F}^{2},1)\|^2_F\} \leq  k_1 + k_2 \|\mat{G}\|^2_F$.

Second, the Lyapunov process $v_t = F(\mat{G}_t) \geq 0$ is proved to converge almost surely. Convergence of $F(\mat{G}_t)$ is then used to show that $w_t=\grad{}{\ F(\mat{G}_t)}$ tends to zero almost surely. Technical details are adapted from the paper of \cite{bottou98a}.
\end{proof}

In practice, saddle points and local maxima are unstable solutions while convergence to asymptotic plateaus is excluded by (A4). As a result, almost sure convergence to a local minimum of the expected cost is obtained.

\vskip 0.2in
\bibliography{meyer11a}
\bibliographystyle{plainnat}

\end{document}